%% file: manuscript.tex
\documentclass[11pt,letterpaper]{article}
\usepackage[margin=1in]{geometry}
\usepackage{amsmath,amssymb,amsthm}
\usepackage{graphicx}
\usepackage{booktabs}
\usepackage{multirow}
\usepackage{microtype}
\usepackage{caption}
\usepackage{placeins}
\usepackage[hidelinks]{hyperref}
\usepackage[capitalise,noabbrev]{cleveref}
\graphicspath{{figures/}}
\input{macros.tex}

\newtheorem{proposition}{Proposition}

\newenvironment{widefigure}{\begin{figure}[tbp]}{\end{figure}}
\newenvironment{widetable}{\begin{table}[tbp]}{\end{table}}
\newcommand{\eqbreak}{\qquad}
\newcommand{\eqbreakq}{\quad}
\newcommand{\eqsplit}{}
\newcommand{\eqsplitdelim}{}

\title{Floquet Fibre Geometry and Higher-Order Reduced Coordinates for
Off-Manifold Transients near Nonlinear Aeroelastic Flutter}
\author{
Puxue Tan\\
\small William Wright Technology Centre (W-Tech)\\
\small School of Mechanical \& Aerospace Engineering\\
\small Queen's University Belfast\\
\small 50 Malone Road, Belfast, BT9 5BS, United Kingdom
}
\date{}

\begin{document}
\maketitle

\begin{abstract}
\input{sections/00_abstract}
\end{abstract}


\input{sections/01_introduction}
\input{sections/02_theory}
\input{sections/03_methods}
\input{sections/04_results}
\FloatBarrier
\input{sections/05_discussion}
\input{sections/06_conclusions}

\section*{Acknowledgements}
The author thanks Emeritus Professor A. J. Roberts for detailed comments on an
earlier version of the manuscript, for drawing attention to the earlier
literature on dynamically appropriate initial conditions and initial slip, and
for suggesting the Hopf-normal-form example used in Section~\ref{sec:firstorder}
to illustrate non-orthogonal isochron projection.

\section*{Data and code availability}
The numerical data, model configurations, trained-model parameters and analysis
code supporting the findings of this study are preserved in a version-controlled
reproducibility archive and are available from the author upon reasonable
request. The archive contains the aeroelastic model, trajectory-generation
procedure, Floquet analysis, coordinate-map definitions and evaluation routines.
All numerical quantities reported in the manuscript are generated
programmatically from machine-readable records.

\section*{Author contributions}
Puxue Tan conceived the study, developed the methodology and software, performed
the formal analysis and numerical investigation, validated the results, prepared
the visualisations, and wrote and revised the manuscript.

\section*{Funding}
This research received no external funding.

\section*{Competing interests}
The author declares no competing interests.

\bibliographystyle{plain}
\bibliography{references}

\end{document}

%% file: macros.tex
\newcommand{\AmplitudeScale}{0.005771}
\newcommand{\AngleVperpWfMax}{71.7}
\newcommand{\AngleVperpWfMedian}{64.3}
\newcommand{\AngleVperpWfMin}{48.5}
\newcommand{\AsixCorrectionScalingData}{2.15}
\newcommand{\AsixCorrectionScalingSynthetic}{2.32}
\newcommand{\AsixFastFiberResidual}{1.478\times 10^{-6}}
\newcommand{\AsixFirstOrderResidual}{1.452\times 10^{-6}}
\newcommand{\AsixLcoCorrection}{1.335\times 10^{-30}}
\newcommand{\AsixParameters}{75}

\newcommand{\BootstrapResamples}{10000}
\newcommand{\BsixParameters}{210}
\newcommand{\BsixSymmetryAmplitudeMax}{8.378\times 10^{-8}}
\newcommand{\BsixValidationImprovement}{23.1}

\newcommand{\CanonicalAsixFivePeriodCiHigh}{-0.113}
\newcommand{\CanonicalAsixFivePeriodCiLow}{-0.517}

\newcommand{\CanonicalAsixFivePeriodRatio}{2.766}
\newcommand{\CanonicalAsixFivePeriodUnits}{5}

\newcommand{\CanonicalAsixOnePeriodCiHigh}{-0.122}
\newcommand{\CanonicalAsixOnePeriodCiLow}{-0.477}

\newcommand{\CanonicalAsixOnePeriodRatio}{2.133}
\newcommand{\CanonicalAsixOnePeriodUnits}{4}

\newcommand{\CanonicalAsixTwoPeriodCiHigh}{-0.017}
\newcommand{\CanonicalAsixTwoPeriodCiLow}{-0.316}

\newcommand{\CanonicalAsixTwoPeriodRatio}{1.746}
\newcommand{\CanonicalAsixTwoPeriodUnits}{5}

\newcommand{\CanonicalBsixFivePeriodCiHigh}{-0.114}
\newcommand{\CanonicalBsixFivePeriodCiLow}{-0.53}

\newcommand{\CanonicalBsixFivePeriodRatio}{2.904}
\newcommand{\CanonicalBsixFivePeriodUnits}{6}

\newcommand{\CanonicalBsixOnePeriodCiHigh}{-0.157}
\newcommand{\CanonicalBsixOnePeriodCiLow}{-0.619}

\newcommand{\CanonicalBsixOnePeriodRatio}{1.777}
\newcommand{\CanonicalBsixOnePeriodUnits}{0}

\newcommand{\CanonicalBsixTwoPeriodCiHigh}{0.004}
\newcommand{\CanonicalBsixTwoPeriodCiLow}{-0.411}

\newcommand{\CanonicalBsixTwoPeriodRatio}{1.297}
\newcommand{\CanonicalBsixTwoPeriodUnits}{7}

\newcommand{\CanonicalPhaseAsixFivePeriodRatio}{1.316}

\newcommand{\CanonicalPhaseAsixOnePeriodRatio}{1.492}

\newcommand{\CanonicalPhaseAsixTwoPeriodRatio}{1.291}

\newcommand{\CanonicalPhaseBsixFivePeriodRatio}{0.97}

\newcommand{\CanonicalPhaseBsixOnePeriodCiHigh}{0.142}
\newcommand{\CanonicalPhaseBsixOnePeriodCiLow}{-0.004}
\newcommand{\CanonicalPhaseBsixOnePeriodPhaseA}{1.26}
\newcommand{\CanonicalPhaseBsixOnePeriodPhaseB}{0.78}
\newcommand{\CanonicalPhaseBsixOnePeriodRatio}{0.821}

\newcommand{\CanonicalPhaseBsixTwoPeriodCiHigh}{0.167}
\newcommand{\CanonicalPhaseBsixTwoPeriodCiLow}{0.011}
\newcommand{\CanonicalPhaseBsixTwoPeriodPhaseA}{1.2}
\newcommand{\CanonicalPhaseBsixTwoPeriodPhaseB}{0.74}
\newcommand{\CanonicalPhaseBsixTwoPeriodRatio}{0.767}

\newcommand{\ConventionAmplitudeMedian}{1.066\times 10^{-4}}
\newcommand{\ConventionPhaseMedian}{8.079\times 10^{-6}}
\newcommand{\CorrectionAsixAmplitudeMedian}{1.309\times 10^{-4}}
\newcommand{\CorrectionAsixPhaseMedian}{5.626\times 10^{-6}}
\newcommand{\CorrectionBsixAmplitudeMedian}{2.191\times 10^{-4}}
\newcommand{\CorrectionBsixPhaseMedian}{1.383\times 10^{-5}}
\newcommand{\CubicPlungeStiffness}{100}
\newcommand{\DataSymmetryResidual}{1.641\times 10^{-12}}
\newcommand{\DecompositionPhases}{512}
\newcommand{\FastDecayFivePeriods}{5.421\times 10^{-5}}
\newcommand{\FastExponent}{-0.343608}
\newcommand{\FastModeClosure}{1.916\times 10^{-11}}

\newcommand{\FloorFirstOrder}{1\times 10^{-5}}
\newcommand{\FloorSymmetry}{4.642\times 10^{-6}}

\newcommand{\FlowDt}{0.1}
\newcommand{\FlutterFrequency}{0.930797}
\newcommand{\FlutterOnsetDiscrepancy}{0.109}
\newcommand{\FlutterVelocity}{7.490584}
\newcommand{\FrameAntiperiodicityResidual}{4.562\times 10^{-11}}
\newcommand{\GaugeAsixAlpha}{1.3599}
\newcommand{\GaugeAsixExplained}{0.9384}
\newcommand{\GaugeAsixGain}{0.7353}

\newcommand{\GaugeAsixRotation}{2.5}
\newcommand{\GaugeBsixAlpha}{1.5033}
\newcommand{\GaugeBsixExplained}{0.8843}
\newcommand{\GaugeBsixGain}{0.6652}

\newcommand{\GaugeBsixRotation}{5.1}

\newcommand{\GaugeFitSamples}{40020}
\newcommand{\GridNodesConvergence}{2048}
\newcommand{\GridNodesPrimary}{4096}

\newcommand{\IntegrationStep}{0.02}
\newcommand{\LcoAngularFrequency}{1.098967}
\newcommand{\LcoPeriod}{5.717358}
\newcommand{\LocalAsixFirstPeriod}{0.815}
\newcommand{\LocalAsixMedian}{0.998}
\newcommand{\LocalAsixSettled}{0.999}
\newcommand{\LocalBsixFirstPeriod}{0.775}
\newcommand{\LocalBsixMedian}{0.996}
\newcommand{\LocalBsixSettled}{0.998}
\newcommand{\MultiplierFast}{0.140222}
\newcommand{\MultiplierNeutral}{1}
\newcommand{\MultiplierSlowImag}{0.773514}
\newcommand{\MultiplierSlowModulus}{0.862784}
\newcommand{\MultiplierSlowReal}{0.382194}
\newcommand{\NeutralMultiplierResidual}{9.758\times 10^{-12}}

\newcommand{\ObliqueConditionMax}{46.73176}

\newcommand{\ObliqueGridConvergenceRms}{4.167\times 10^{-10}}
\newcommand{\ObliqueLcoAmplitude}{1.131\times 10^{-9}}
\newcommand{\ObliqueLcoPhaseError}{1.11\times 10^{-10}}
\newcommand{\ObliqueSymmetryAmplitudeMax}{1.484\times 10^{-7}}
\newcommand{\ObliqueTrainSlopeInvariance}{2}
\newcommand{\ObliqueValidationSlopeInvariance}{2.05}

\newcommand{\OperatingVelocity}{7.8}

\newcommand{\OrthogonalLcoPhaseError}{1.06\times 10^{-10}}

\newcommand{\PairedAsixObliqueFivePeriodInverse}{1.68}

\newcommand{\PairedAsixObliqueInvarianceInverse}{0.8}

\newcommand{\PairedAsixObliqueOnePeriodInverse}{0.79}

\newcommand{\PairedBsixObliqueFivePeriodInverse}{1}

\newcommand{\PairedBsixObliqueInvarianceInverse}{0.55}

\newcommand{\PairedBsixObliqueOnePeriodInverse}{0.51}

\newcommand{\PerturbationCalibration}{3.331\times 10^{-16}}
\newcommand{\PhaseCount}{16}
\newcommand{\PilotTrajectories}{120}
\newcommand{\PilotWorstTerminalDistance}{2.231\times 10^{-4}}

\newcommand{\RestoredAsixFivePeriodCiHigh}{-0.037}
\newcommand{\RestoredAsixFivePeriodCiLow}{-0.501}

\newcommand{\RestoredAsixFivePeriodRatio}{2.267}
\newcommand{\RestoredAsixFivePeriodUnits}{5}

\newcommand{\RestoredAsixInvarianceCiHigh}{0.096}
\newcommand{\RestoredAsixInvarianceCiLow}{-0.118}

\newcommand{\RestoredAsixInvariancePhaseADeltaA}{1.04}
\newcommand{\RestoredAsixInvariancePhaseADeltaB}{1.13}
\newcommand{\RestoredAsixInvariancePhaseADeltaC}{1.15}
\newcommand{\RestoredAsixInvariancePhaseADeltaD}{1.14}
\newcommand{\RestoredAsixInvariancePhaseADeltaE}{1.17}
\newcommand{\RestoredAsixInvariancePhaseBDeltaA}{1.62}
\newcommand{\RestoredAsixInvariancePhaseBDeltaB}{1.5}
\newcommand{\RestoredAsixInvariancePhaseBDeltaC}{1.06}
\newcommand{\RestoredAsixInvariancePhaseBDeltaD}{0.78}
\newcommand{\RestoredAsixInvariancePhaseBDeltaE}{0.65}
\newcommand{\RestoredAsixInvarianceRatio}{1.084}
\newcommand{\RestoredAsixInvarianceUnits}{7}

\newcommand{\RestoredAsixOnePeriodCiHigh}{0.098}
\newcommand{\RestoredAsixOnePeriodCiLow}{-0.135}

\newcommand{\RestoredAsixOnePeriodRatio}{1.071}
\newcommand{\RestoredAsixOnePeriodUnits}{10}

\newcommand{\RestoredBsixFivePeriodCiHigh}{0.05}
\newcommand{\RestoredBsixFivePeriodCiLow}{-0.32}

\newcommand{\RestoredBsixFivePeriodRatio}{1.483}
\newcommand{\RestoredBsixFivePeriodUnits}{8}

\newcommand{\RestoredBsixInvarianceCiHigh}{0.225}
\newcommand{\RestoredBsixInvarianceCiLow}{-0.044}

\newcommand{\RestoredBsixInvariancePhaseADeltaA}{1.36}

\newcommand{\RestoredBsixInvariancePhaseADeltaE}{0.81}
\newcommand{\RestoredBsixInvariancePhaseBDeltaA}{1.96}

\newcommand{\RestoredBsixInvariancePhaseBDeltaE}{0.49}
\newcommand{\RestoredBsixInvarianceRatio}{0.819}
\newcommand{\RestoredBsixInvarianceUnits}{13}

\newcommand{\RestoredBsixOnePeriodCiHigh}{0.291}
\newcommand{\RestoredBsixOnePeriodCiLow}{-0.121}

\newcommand{\RestoredBsixOnePeriodRatio}{0.758}
\newcommand{\RestoredBsixOnePeriodUnits}{12}

\newcommand{\RidgeWeight}{1\times 10^{-6}}
\newcommand{\SampleInterval}{0.1}
\newcommand{\ScaleAsixAnisotropy}{1.23}

\newcommand{\ScaleAsixRms}{0.885}
\newcommand{\ScaleBsixAnisotropy}{1.27}

\newcommand{\ScaleBsixRms}{0.777}
\newcommand{\ScaleChartAnisotropy}{1.29}

\newcommand{\ScaleChartRms}{1.242}
\newcommand{\ScaleFractionAsixInvariance}{0.46}
\newcommand{\ScaleFractionAsixOnePeriod}{0.55}
\newcommand{\ScaleFractionBsixInvariance}{0.48}
\newcommand{\ScaleFractionBsixOnePeriod}{0.57}
\newcommand{\SeamClosureResidual}{1.449\times 10^{-11}}

\newcommand{\SlopeAsixInvariance}{1.71}

\newcommand{\SlopeBsixInvariance}{1.33}

\newcommand{\SlopeObliqueInvariance}{1.87}
\newcommand{\SlopeObliqueOnePeriod}{1.87}

\newcommand{\SlopeOrthogonalInvariance}{1.01}
\newcommand{\SlopeOrthogonalOnePeriod}{0.98}
\newcommand{\SlowFrameConditionMax}{22.24258}
\newcommand{\StepCoarse}{0.04}
\newcommand{\StepConvergenceDeviation}{5.877\times 10^{-5}}
\newcommand{\StepFine}{0.01}

\newcommand{\TestDistinctPhases}{2}

\newcommand{\TestObliqueInvariance}{0.009145}
\newcommand{\TestObliqueOnePeriod}{0.233853}

\newcommand{\TestOffManifoldUnits}{20}

\newcommand{\TestOrthogonalInvariance}{0.031828}
\newcommand{\TestOrthogonalOnePeriod}{0.426594}
\newcommand{\TestSymmetryUnits}{22}
\newcommand{\TestTrajectories}{44}
\newcommand{\TrainBatchSize}{2048}
\newcommand{\TrainIterations}{6000}
\newcommand{\TrainLearningRate}{0.001}
\newcommand{\TrajTest}{44}
\newcommand{\TrajTrain}{56}
\newcommand{\TrajValidation}{20}
\newcommand{\TrajectoryDuration}{200}

\newcommand{\VperpFastCoefficientMedian}{2.36}
\newcommand{\VperpPhaseCoefficientMax}{36.5}
\newcommand{\VperpPhaseCoefficientMedian}{31.4}
\newcommand{\VperpPhaseCoefficientMin}{11.8}
\newcommand{\VperpPhaseShiftLargestDelta}{0.63}

\newcommand{\VperpSlowNormMedian}{2.08}
\newcommand{\WhitenedAsixFivePeriodRatio}{2.267}
\newcommand{\WhitenedAsixInvarianceRatio}{1.135}
\newcommand{\WhitenedAsixOnePeriodRatio}{1.111}
\newcommand{\WhitenedBsixFivePeriodRatio}{1.613}
\newcommand{\WhitenedBsixInvarianceRatio}{0.893}
\newcommand{\WhitenedBsixOnePeriodRatio}{0.814}

%% file: sections/00_abstract.tex
Assigning a reduced coordinate to an off-manifold state near an attracting
limit cycle is governed by invariant-fibre geometry and is not determined by
the reduced dynamics alone. The classical first-order choice is the linearised phase--isostable
chart, which projects along the fast Floquet direction tangent to the local
strong-stable fibre; a metric-orthogonal complement of the retained slow bundle
is easier to construct but coincides with that direction only by accident. We
make the consequence explicit in a local proposition: a chart satisfying the
linearised semiconjugacy relation leaves an invariance residual of second order
in the transverse amplitude $\delta$, whereas projection along a non-invariant
complement generically leaves a first-order term. On the limit cycle of a
nonlinear aeroelastic typical section beyond flutter, the two directions differ
by $\AngleVperpWfMin^{\circ}$ to $\AngleVperpWfMax^{\circ}$. In a test
specified before the held-out data were opened, replacing the metric-normal
direction by the fast Floquet direction changes the observed residual scaling
from $\delta^{\SlopeOrthogonalInvariance}$ to $\delta^{\SlopeObliqueInvariance}$,
with no fitted parameters. Learned higher-order corrections, pinned to this
chart and sharing a frozen reduced flow, reduce the pre-specified latent
residual on held-out data; two post-hoc diagnostics move or reverse that
ranking, but both remain convention dependent: one through coordinate
normalisation and the other through the reference chart. For this
benchmark, correct first-order Floquet geometry is necessary to remove the
generic $O(\delta)$ semiconjugacy defect under the frozen reduced dynamics,
whereas the additional value of the learned higher-order correction is not
identifiable from the available diagnostics.

%% file: sections/01_introduction.tex
\section{Introduction}
\label{sec:introduction}

Beyond the flutter boundary a nonlinear aeroelastic section does not simply
diverge. Structural or aerodynamic nonlinearity arrests the linear instability
and the system settles onto a limit cycle whose amplitude and period are set by
the nonlinearity, not by the initial disturbance
\cite{fung1993aeroelasticity,dowell2022modern}. A reduced-order description
meant to be useful after flutter onset therefore has to describe motion
\emph{near} an attracting periodic orbit instead of near an equilibrium.

Such a description needs two objects: a reduced map that advances the reduced
state, and a reduction map that sends a full state to its reduced coordinate.
For a trajectory already on the attractor the second object is almost trivial.
For a state displaced off the attractor it is not, and it is the reduction map
that decides whether the reduced description remains closed: full states
carrying the same reduced coordinate must have reduced futures that agree, or
the reduced map cannot be a function of the reduced state alone.

Spectral submanifolds and their parameterisations supply the first object,
providing low-dimensional invariant sets that carry the slow dynamics
\cite{haller2016ssm,cabre2003parameterization1,cenedese2022datadriven,jain2024ssmtool};
they describe motion \emph{on} a manifold. The complementary object, an
invariant foliation, partitions a neighbourhood into leaves that the flow maps
into one another, so that states on a common leaf share an asymptotic reduced
future \cite{szalai2020isf,szalai2023rom,hirsch1977invariant}. It therefore
defines the quotient that a reduction map should realise.

Near a stable hyperbolic periodic orbit both objects are classical, as is the
relation between them. The leaves of the isochron foliation are the level sets
of the asymptotic phase, and the isostable coordinates complete them in the
transverse directions; together they form the phase--isostable coordinates of
the limit cycle. Their linearisation, which we call the linear or
adjoint-Floquet chart, is read directly from the adjoint Floquet modes: the
phase sensitivity is the infinitesimal phase response curve obtained from the
adjoint equation \cite{brown2004phase}, and the amplitude coordinates follow
from the Floquet normal form \cite{wilson2016isostable}. Phase and isostable
coordinates admit a Koopman-eigenfunction representation and extend away from
the orbit \cite{mauroy2013isostables,shirasaka2017transient}: the complex phase
observable $e^{i\theta}$ is a Koopman eigenfunction, while the retained complex
amplitude coordinate is associated with the corresponding Floquet exponent.
They are unique under
suitable regularity and spectral conditions, up to the usual normalisation and
coordinate freedoms \cite{kvalheim2021existence}, and they can be computed to
high order by the parameterization method \cite{perezcervera2020global} or
fitted from data \cite{wilson2020datadriven,ahmed2023isostable}; invariant
foliations more generally can also be fitted to trajectory data
\cite{szalai2025datadriven}. The projection defining the linear chart is
oblique. It discards the fast Floquet direction, which is tangent at the orbit
to the local strong-stable fibre and is not orthogonal to the slow bundle in any
metric chosen for convenience.

How an off-manifold state should be assigned its reduced counterpart is an old
question. It arose with centre- and slow-manifold models, which need initial
conditions consistent with the full dynamics. Roberts treated it explicitly,
determining the projection of initial conditions along what he termed
isochronic manifolds and constructing it by normal forms and by computer
algebra, including near a Hopf bifurcation
\cite{roberts1989appropriate,cox1995initial,roberts2000computer,roberts2015emergent};
for limit cycles the corresponding leaves are the isochrons
\cite{guckenheimer1975isochrons}. Recent spectral-submanifold work has returned
to the issue in data-driven reduction at equilibria, showing that normal
projection mis-pairs off-manifold states with their reduced counterparts and
fitting an oblique projection as a practical surrogate for the stable foliation
\cite{bettini2025oblique,bettini2026general}.

What is not established is how much that distinction costs for a nonlinear
aeroelastic limit cycle, where a plausible metric-normal complement is easy to
adopt and its consequences are invisible on the orbit itself. Nor is it known
what a learned nonlinear correction adds once the linear chart is in place. The
linear construction is exact only to first order, so a learned term could
matter as soon as the excursion is large enough for leaf curvature to be felt;
equally, the linear chart may already suffice within the amplitude range in
which a three-dimensional description is useful at all. Answering the second
question requires holding the reduced dynamics fixed while only the coordinate
varies, and it requires a figure of merit that does not merely reflect the
units in which a reduced coordinate happens to be returned. That second
requirement turned out to be the harder one, and it is why the two questions
separate so sharply below.

The contributions are as follows.

\begin{enumerate}
\item An explicit local statement (Proposition~\ref{prop:firstorder}) of a
  mechanism that follows from established theory. A reduced coordinate whose
  linearisation intertwines the variational flow with the linearised reduced
  map leaves an $O(\delta^{2})$ residual in the invariance relation, whereas a
  projection along a complement that is not the invariant discarded direction
  generically leaves an $O(\delta)$ term, whose coefficient is given in closed
  form for the metric-normal complement. The statement is local, assumes the
  stated smoothness, and is explicit about the accidental cancellations that
  would make the linear coefficient vanish.
\item A quantification of the discrepancy on the nonlinear aeroelastic limit
  cycle: how far the metric-orthogonal complement lies from the fast Floquet
  direction, and how much retained phase and slow-amplitude content a
  metric-normal displacement therefore carries. Both retained components are
  present; which one dominates a measured residual depends on the coordinate
  normalisation and is not a property of the geometry alone.
\item A controlled numerical test of the predicted scaling on held-out
  trajectories, specified before those data were opened. The adjoint-Floquet
  construction itself is classical
  \cite{wilson2016isostable,brown2004phase}; what is tested is its consequence
  for this system.
\item An identifiability audit of the higher-order question. We construct
  learned corrections whose linear part is pinned to the classical chart, whose
  contribution begins structurally at second order, which respect the discrete
  symmetry exactly and which share a frozen reduced flow. Their interpretation
  depends materially on the latent error metric, on a validation-fitted
  amplitude recalibration, on the reference chart used to encode future states,
  on phase sampling and on the anchor convention, and the study identifies
  neither a robust higher-order benefit nor a robust higher-order null.
\end{enumerate}

The evidence is confined to one benchmark at one operating point over one
family of transverse perturbations; Section~\ref{sec:limitations} sets out what
this allows us to conclude.

%% file: sections/02_theory.tex
\section{Reduction geometry near a stable limit cycle}
\label{sec:theory}

\subsection{Dynamical consistency and first-order semiconjugacy}
\label{sec:semiconj}

Let $\dot{x}=F(x)$, $x\in\mathbb{R}^{n}$, have flow $\varphi^{t}$ and possess an
attracting hyperbolic limit cycle $\gamma$ of period $T$, parameterised by phase
$\theta\in S^{1}$ with $\gamma(\theta+2\pi)=\gamma(\theta)$ and
$\gamma'(\theta)=F(\gamma(\theta))/\omega$, $\omega=2\pi/T$, so that
$\varphi^{t}(\gamma(\theta))=\gamma(\theta+\omega t)$. Write $\Psi(t)$ for the
fundamental solution of the variational equation
$\dot{\Psi}=DF(\gamma(\omega t))\Psi$, $\Psi(0)=I$, whose monodromy matrix
$\Psi(T)$ has the Floquet multipliers as eigenvalues, and
\begin{equation}
A_{\tau}(\theta)=D\varphi^{\tau}\bigl(\gamma(\theta)\bigr)
 =\Psi\!\left(\tfrac{\theta}{\omega}+\tau\right)
  \Psi\!\left(\tfrac{\theta}{\omega}\right)^{-1}
\label{eq:propagator}
\end{equation}
for the variational propagator over time $\tau$ from the orbit point of phase
$\theta$.

A reduced coordinate, or reduction map, is a map $u:\mathcal{U}\to Z$ from a
neighbourhood $\mathcal{U}$ of $\gamma$ to a reduced state space $Z$, dynamically consistent
for the time-$\tau$ map $\varphi^{\tau}$ if there is a reduced map $r_{\tau}$
with
\begin{equation}
u\bigl(\varphi^{\tau}(x)\bigr)=r_{\tau}\bigl(u(x)\bigr)
\qquad\text{for all } x\in\mathcal{U}.
\label{eq:invariance}
\end{equation}
Equation~\eqref{eq:invariance} is the defining property of an invariant
foliation: its level sets are the leaves, and the flow maps leaves onto leaves
\cite{szalai2020isf,hirsch1977invariant}. For a stable hyperbolic limit cycle
the classical instance is the pair of isochron and isostable foliations, whose
leaves are level sets of the asymptotic phase and of the slowly decaying
Koopman eigenfunctions
\cite{guckenheimer1975isochrons,mauroy2013isostables,kvalheim2021existence}.

Three quantities are easily conflated, and the distinction governs what the
experiments below can and cannot decide.

\begin{enumerate}
\item The \emph{coordinate-assignment error}: the discrepancy between $u(x)$
  and the exact nonlinear phase--isostable coordinate of $x$. Measuring it
  requires an independently computed reference coordinate.
\item The \emph{semiconjugacy residual}
  $R^{\tau}_{u}(x)=u(\varphi^{\tau}(x))-r_{\tau}(u(x))$, the defect in
  \eqref{eq:invariance}; evaluated over one sampling interval it is the
  one-step invariance residual used in the experiments. It needs no reference
  coordinate, but it is measured in whatever units $u$ itself returns.
\item The \emph{finite-horizon prediction error}, obtained by iterating
  $r_{\tau}$, which depends additionally on the chart in which the future
  state is expressed.
\end{enumerate}

Only the second and third are observable in this study. Both depend on how the
reduced coordinate is normalised (Section~\ref{sec:gauge}), and
Section~\ref{sec:discussion} returns to what follows from that.

The following proposition makes the local mechanism explicit in general form.

\begin{proposition}[First-order semiconjugacy defect]
\label{prop:firstorder}
Let $F$ be $C^{2}$ on a neighbourhood of $\gamma$, let $u$ be $C^{2}$ and
$r_{\tau}$ be $C^{2}$ on a neighbourhood of $u(\gamma)$, and suppose
\eqref{eq:invariance} holds on the orbit, so that
$R^{\tau}_{u}(\gamma(\theta))=0$ for every $\theta$. Write
$P(\theta)=Du(\gamma(\theta))$ and
$R_{\tau}(\theta)=Dr_{\tau}(u(\gamma(\theta)))$, fix $\theta$, and put
$x_{\delta}=\gamma(\theta)+\delta v$ for a unit vector $v$. Then
\begin{equation}
\begin{gathered}
R^{\tau}_{u}(x_{\delta})=\delta\,K_{\tau}(\theta)\,v+O(\delta^{2}),
\eqbreak
K_{\tau}(\theta)=P(\theta+\omega\tau)\,A_{\tau}(\theta)-R_{\tau}(\theta)\,P(\theta),
\end{gathered}
\label{eq:firstordercoeff}
\end{equation}
with the remainder locally uniform in $\theta$ on the compact orbit. If $u$
satisfies the linearised semiconjugacy relation
\begin{equation}
P(\theta+\omega\tau)\,A_{\tau}(\theta)=R_{\tau}(\theta)\,P(\theta)
\qquad\text{for every }\theta,
\label{eq:semiconj}
\end{equation}
then $R^{\tau}_{u}(x_{\delta})=O(\delta^{2})$. Otherwise the residual is exactly
first order in every direction that $K_{\tau}(\theta)$ does not annihilate: if
$K_{\tau}(\theta)v\neq0$ then
\begin{equation}
\bigl\lVert R^{\tau}_{u}(x_{\delta})\bigr\rVert=\Theta(|\delta|),
\label{eq:firstorderrate}
\end{equation}
while if $K_{\tau}(\theta)v=0$ an accidental directional cancellation occurs
and the proposition makes no first-order claim for that direction.
\end{proposition}

\begin{proof}
Both $x\mapsto u(\varphi^{\tau}(x))$ and $x\mapsto r_{\tau}(u(x))$ are $C^{2}$
near $\gamma(\theta)$: the time-$\tau$ flow of a $C^{2}$ vector field is $C^{2}$
in its initial condition, and the compositions are $C^{2}$ by the chain rule,
given $u,r_{\tau}\in C^{2}$. Their second derivatives are therefore continuous,
hence bounded on a compact neighbourhood of $\gamma$, and Taylor's theorem with
remainder gives an expansion of each about $\gamma(\theta)$ along $v$ whose
remainder is $O(\delta^{2})$ with a constant controlled by that uniform bound;
the compactness of $\gamma$ makes the bound, and so the remainder, uniform in
$\theta$. Subtracting the two expansions gives \eqref{eq:firstordercoeff}: the
zeroth-order terms cancel because the residual vanishes on the orbit, and the
chain rule with $\varphi^{\tau}(\gamma(\theta))=\gamma(\theta+\omega\tau)$ gives
the form of $K_{\tau}$. If $K_{\tau}(\theta)$ vanishes identically, which is
what \eqref{eq:semiconj} asserts, the linear term is absent and only the
$O(\delta^{2})$ remainder survives. If instead $K_{\tau}(\theta)v\neq0$, the
linear term dominates the remainder for all sufficiently small $|\delta|$,
which gives \eqref{eq:firstorderrate}.
\end{proof}

Its centre-manifold counterpart, that projecting initial conditions along the
isochron tangents leaves an error of second order in the distance from the
manifold, is Theorem~12.4 of \cite{roberts2015emergent}; the proposition states
the corresponding mechanism for the semiconjugacy residual near a periodic
orbit. It is local in $\delta$ and says nothing about states outside the
neighbourhood on which the expansion holds. Because the phase coordinate lives
on $S^{1}$, the Taylor expansions and the subtraction defining
$R^{\tau}_{u}$ are to be read in a local lift of that circle to $\mathbb{R}$,
which is legitimate precisely because the statement is local in $\delta$: for
small enough displacement all the phases involved lie in one chart, and the
residual is the difference taken there. The reading of the second case also
needs care. Failure of \eqref{eq:semiconj} makes the residual first order
generically rather than always, since a particular displacement direction may
lie in the kernel of $K_{\tau}(\theta)$; where it does not,
\eqref{eq:firstorderrate} gives an exact first-order rate rather than only an
upper bound.

\subsection{Floquet geometry and phase--isostable coordinates}
\label{sec:firstorder}

Specialising, suppose the multipliers split into a slow group retained and a
fast group discarded. For the benchmark below the split is one neutral
direction, one stable complex pair and one faster stable real direction, so the
slow bundle is three-dimensional and its complement one-dimensional. Let
$\lambda_{c}$ be the exponent of the retained complex pair with eigenvector
$z_{c}$ of $\Psi(T)$, and $\lambda_{f}$, $w_{f}(0)$ those of the discarded real
mode. Propagating with the exponential factor removed gives $T$-periodic
Floquet modes,
\begin{equation}
\begin{gathered}
v_{c}(\theta)=\Psi\!\left(\tfrac{\theta}{\omega}\right)z_{c}\,
  e^{-\lambda_{c}\theta/\omega},
\eqbreak
w_{f}(\theta)=\Psi\!\left(\tfrac{\theta}{\omega}\right)w_{f}(0)\,
  e^{-\lambda_{f}\theta/\omega},
\end{gathered}
\label{eq:floquetmodes}
\end{equation}
and the real primal Floquet frame
\begin{equation}
B_{F}(\theta)=\bigl[\,\gamma'(\theta)\;\;
  \mathrm{Re}\,v_{c}(\theta)\;\;\mathrm{Im}\,v_{c}(\theta)\;\;
  w_{f}(\theta)\,\bigr].
\label{eq:frame}
\end{equation}
The first three columns span the retained slow bundle, denoted $E^{s}(\theta)$,
where the superscript $s$ denotes slow rather than stable (the bundle contains
the neutral phase direction), and the fourth the fast Floquet subbundle. Under the usual smoothness and exponential-separation
assumptions for the local strong-stable foliation of a hyperbolic periodic
orbit \cite{hirsch1977invariant,guckenheimer1975isochrons}, that subbundle
integrates locally to strong-stable fibres, so $w_{f}(\theta)$ is tangent at
$\gamma(\theta)$ to the fibre through that point, the curved invariant leaf along
which nearby states approach the orbit fastest. The first-order construction
below projects along this tangent, the fast Floquet direction, and does not
reconstruct the nonlinear fibre away from the orbit. Inverting the primal frame gives the
dual basis $B_{F}^{-1}(\theta)$, whose rows are the corresponding adjoint
Floquet covectors. The variational flow carries the frame into itself,
\begin{equation}
\begin{gathered}
A_{\tau}(\theta)\,B_{F}(\theta)=B_{F}(\theta+\omega\tau)\,\Lambda_{\tau},
\eqbreak
\Lambda_{\tau}=\mathrm{diag}\bigl(\Lambda^{s}_{\tau},\,e^{\lambda_{f}\tau}\bigr),
\end{gathered}
\label{eq:floquetcocycle}
\end{equation}
where $\Lambda^{s}_{\tau}$ is the linear part of the block-diagonal reduced map
of Section~\ref{sec:reducedflow}. Taking the first three rows of
$B_{F}^{-1}(\theta+\omega\tau)A_{\tau}(\theta)=\Lambda_{\tau}B_{F}^{-1}(\theta)$
shows that $P_{F}(\theta)$, the first three rows of $B_{F}^{-1}(\theta)$,
satisfies \eqref{eq:semiconj} with $R_{\tau}=\Lambda^{s}_{\tau}$, so
Proposition~\ref{prop:firstorder} gives an $O(\delta^{2})$ residual for a
reduced coordinate with that linearisation. We call such a map the
\emph{adjoint-Floquet projection} $u_{\mathrm{aF}}$: it is the linearised
phase--isostable chart of the limit cycle, in which the phase row of
$B_{F}^{-1}$ is the infinitesimal phase response curve obtained from the
adjoint equation \cite{brown2004phase} and the remaining rows are the linear
isostable coordinates of the retained Floquet pair \cite{wilson2016isostable}.
We use it as established theory. The projection is oblique: $w_{f}$ is in
general not orthogonal to $E^{s}$ in any metric chosen for convenience.

Given a positive-definite metric $M$ one may instead form the $M$-orthogonal,
or metric-normal, complement $v_{\perp}(\theta)$ of $E^{s}(\theta)$ and project
along that. Both
constructions annihilate a one-dimensional direction and both are exact on
$\gamma$, so the substitution is easy to make and hard to detect on the orbit.
Decomposing a unit $v_{\perp}$ in the frame \eqref{eq:frame},
\begin{equation}
\begin{gathered}
v_{\perp}(\theta)=c_{\theta}(\theta)\gamma'(\theta)
  +c_{a}(\theta)\mathrm{Re}\,v_{c}(\theta)
\eqsplit
  +c_{b}(\theta)\mathrm{Im}\,v_{c}(\theta)
  +c_{f}(\theta)w_{f}(\theta),
\end{gathered}
\label{eq:vperpdecomp}
\end{equation}
the retained coefficients $\hat{c}=(c_{\theta},c_{a},c_{b})$ vanish only if the
two complements coincide. The map $u_{\mathrm{orth}}$ that annihilates
$v_{\perp}$ has a linearisation $P_{\perp}(\theta)$ on $\gamma$ that agrees with
$P_{F}(\theta)$ on $E^{s}(\theta)$ and has kernel $v_{\perp}(\theta)$. Along the
direction tested below, \eqref{eq:floquetcocycle} then gives
\begin{equation}
\begin{gathered}
K^{\perp}_{\tau}(\theta)\,v_{\perp}(\theta)
 =\Lambda^{s}_{\tau}\,\hat{c}(\theta)
\eqsplit
  -e^{\lambda_{f}\tau}\,\frac{c_{f}(\theta)}{c_{f}(\theta+\omega\tau)}\,
   \hat{c}(\theta+\omega\tau),
\end{gathered}
\label{eq:orthcoeff}
\end{equation}
the retained content propagated by the slow dynamics, less the retained content
that $u_{\mathrm{orth}}$ assigns to the propagated fast component at the
advanced phase. It vanishes when $v_{\perp}\parallel w_{f}$ and otherwise only
where three scalar conditions happen to hold, so the metric-normal map
generically leaves an $O(\delta)$ residual. Along $v_{\perp}$ the residual of
$u_{\mathrm{orth}}$ is therefore predicted to scale with exponent one and that
of $u_{\mathrm{aF}}$ with exponent two, which is the contrast
Section~\ref{sec:geometry} tests.

Two features of \eqref{eq:vperpdecomp} matter for what follows. First, the
retained content is in general \emph{both} a phase component and a
slow-amplitude component: $c_{\theta}$ is the infinitesimal phase response
curve evaluated along $v_{\perp}$, so states at $\pm\delta v_{\perp}$ lie on
different isochrons, while $(c_{a},c_{b})$ displaces the retained complex
amplitude. There is accordingly no general theorem that the leading defect for
a periodic orbit is a phase error, and we do not claim one; what is specific to
a limit cycle is only that one retained direction is neutral, so the phase part
of a mis-assignment is not contracted by the linear flow. Second, which
retained component dominates a measured residual is normalisation dependent:
$c_{\theta}$, $c_{a}$ and $c_{b}$ multiply basis vectors with different
physical normalisations and cannot be ranked by raw magnitude alone, and the
residual of Section~\ref{sec:metrics} weights them through a particular choice
of phase and amplitude units stated there. The mechanism itself is the
periodic-orbit counterpart of the mis-pairing that normal projection produces
for slow spectral submanifolds at equilibria \cite{bettini2025oblique}.

\paragraph{A two-dimensional illustration.}
The misalignment is not peculiar to the benchmark. In the Hopf normal form
\begin{equation}
\dot{z}=(\sigma+i\nu)\,z-g\,(1-i\eta)\,\lvert z\rvert^{2}z,
\qquad \sigma,g>0,
\label{eq:hopfnf}
\end{equation}
the polar form $\dot{r}=\sigma r-gr^{3}$, $\dot{\psi}=\nu+\eta g r^{2}$ of
$z=re^{i\psi}$ has the attracting cycle $r_{\ast}=(\sigma/g)^{1/2}$ with frequency
$\nu+\eta\sigma$, and
\begin{equation}
\vartheta(r,\psi)=\psi+\eta\ln(r/r_{\ast})
\label{eq:hopfphase}
\end{equation}
is its asymptotic phase, since $\dot{\vartheta}=\nu+\eta\sigma$ everywhere; the
isochrons are logarithmic spirals. At the cycle the fast Floquet exponent is
$-2\sigma$ and its eigendirection, which is also the isochron tangent, is
$e_{r}-\eta\,e_{\psi}$ in the polar unit vectors, whereas the Euclidean normal to
the circle is $e_{r}$. The two make an angle $\beta$ with
$\tan\beta=\lvert\eta\rvert$. In this normal-form parametrisation the angle is
independent of $\sigma$ and $\nu$: it is set by the amplitude dependence of the
frequency, not by the distance from the bifurcation or the relaxation rate.
Nearest-point projection therefore assigns a state at
radius $r_{\ast}+\delta$ a phase error $\eta\,\delta/r_{\ast}+O(\delta^{2})$ that
the neutral phase dynamics never removes, the persistent phase discrepancy that
a poor initial condition produces near a Hopf bifurcation
\cite{roberts2000computer}. Cubic stiffness of the kind in \eqref{eq:eom} can
produce an amplitude-dependent oscillation frequency, which provides the
qualitative analogue of $\eta\neq0$; we do not compute normal-form coefficients
for the benchmark, and the example is illustrative only.

\subsection{Higher-order correction}

Beyond first order the leaves of a genuine invariant foliation are curved and
the linear chart is no longer exact. The exact nonlinear phase--isostable
coordinates can be computed to high order by the parameterization method for
systems given in closed form \cite{perezcervera2020global} and have been fitted
from data \cite{wilson2020datadriven,ahmed2023isostable}, as have invariant
foliations more generally \cite{szalai2023rom}. The object to learn is the
remainder
\begin{equation}
\begin{gathered}
u_{\mathrm{IF}}(x)=u_{\mathrm{aF}}(x)+\Delta u(x),
\eqbreak
\Delta u(x)=O\!\left(\lVert x-\gamma(\theta^{\ast})\rVert^{2}\right),
\end{gathered}
\label{eq:correction}
\end{equation}
with $\theta^{\ast}$ the anchor phase, the phase of the reference point on
$\gamma$ assigned to $x$ (Section~\ref{sec:methods}). What matters for the
comparison is that
$\Delta u$ vanish together with its first derivative on $\gamma$, so that the
correction cannot repair or silently replace a wrong linear projection, and
that $u_{\mathrm{IF}}$ inherit the on-orbit fibre condition
$Du(\gamma(\theta))\,w_{f}(\theta)=0$ from the base chart, so that the two maps
differ only in curvature; Section~\ref{sec:methods} enforces both structurally.
Both are statements on the orbit. They pin the linearisation of the learned map
to $Du_{\mathrm{aF}}$ there, and they do not assert that the network
reconstructs an exact nonlinear invariant fibre through every nearby state. We
call the learned object a \emph{learned higher-order phase--isostable
correction}, and $u_{\mathrm{IF}}$, the subscript standing for invariant
foliation, an approximate data-driven quotient coordinate.

%% file: sections/03_methods.tex
\section{Aeroelastic benchmark and methodology}
\label{sec:methodology}

\subsection{Benchmark and operating point}
\label{sec:benchmark}

\paragraph{Equations and parameters.}
We use the isolated single-airfoil subsystem introduced in Section~2 of Nitti
et al. \cite{nitti2021localized} as the building block of their rotor study,
taken here as a stand-alone benchmark rather than as a model of the coupled
rotor. It is a two-degree-of-freedom typical section with plunge
$h$ (positive downward) and pitch $\alpha$ (positive nose-up) under
quasi-steady aerodynamics, with cubic stiffness in both degrees of freedom.
With $q=[h,\alpha]^{\top}$ and nondimensional time $t=t_{\rm dim}\omega_\alpha$,
where $\omega_{\alpha}$ is the reference frequency of the source,
\begin{equation}
\mathcal{M}\ddot{q}+\mathcal{C}(V)\dot{q}+\mathcal{K}_{1}(V)q+\mathcal{K}_{3}q^{\circ 3}=0,
\label{eq:eom}
\end{equation}
where $q^{\circ 3}$ is the componentwise cube and
\begin{equation}
\begin{gathered}
\mathcal{M}=\begin{bmatrix}1&\varepsilon\\ \varepsilon&r^{2}\end{bmatrix},
\eqbreakq
\mathcal{C}=\begin{bmatrix}
 \mu_{h}+\xi_{u}V & \xi_{u}V(\tfrac12-a)\\[2pt]
 -\xi_{u}V(\tfrac12-a) & \mu_{\alpha}-\xi_{u}V(\tfrac14-a^{2})\end{bmatrix},
\eqbreakq
\mathcal{K}_{1}=\begin{bmatrix}
 \xi_{h0} & \xi_{u}V^{2}\\[2pt]
 0 & \xi_{\alpha0}-\xi_{u}V^{2}(\tfrac12+a)\end{bmatrix},
\end{gathered}
\label{eq:matrices}
\end{equation}
with $\mathcal{K}_{3}=\mathrm{diag}(\xi_{h3},\xi_{\alpha3})$. The state is
$x=[h,\alpha,\dot{h},\dot{\alpha}]^{\top}\in\mathbb{R}^{4}$ and $V$ is the
reduced velocity. The quasi-steady effective-angle closure of the source
carries no lag states, so \eqref{eq:eom} is a four-dimensional autonomous
system. The matrices are those of the source's matrix form (its Eq.~(11)); the
$(2,1)$ entry of $\mathcal{C}$ implied by its scalar equations of motion (its
Eq.~(9)) has $\tfrac12+a$ in place of $\tfrac12-a$, and we use the matrix form
throughout. Parameter values follow Table~1 of \cite{nitti2021localized},
except the cubic plunge coefficient, which is set to
$\xi_{h3}=\CubicPlungeStiffness$, one of the values in the source's study of
that coefficient, in place of the base value of its Table~1. The values and
the symbols in \eqref{eq:matrices} are listed in Table~\ref{tab:parameters}.

\begin{table}[htbp]
\centering
\caption{Benchmark parameters and operating point. Structural and aerodynamic
values are those of Table~1 of \cite{nitti2021localized}, except $\xi_{h3}$,
which takes one of the values of the source's study of the cubic plunge
coefficient in place of the Table-1 base value; the remaining quantities are
computed from \eqref{eq:eom} and recorded with the computational provenance of
each run.}
\label{tab:parameters}
\begin{tabular}{llr}
\toprule
Symbol & Meaning & Value \\
\midrule
$\varepsilon$      & static unbalance                 & $0.25$ \\
$r^{2}$            & squared radius of gyration        & $0.5$ \\
$a$                & elastic-axis offset               & $-0.1$ \\
$\xi_{u}$          & aerodynamic coefficient           & $0.0113$ \\
$\mu_{h},\mu_{\alpha}$ & structural damping            & $0.1$ \\
$\xi_{h0}$         & linear plunge stiffness           & $0.2$ \\
$\xi_{\alpha0}$    & linear pitch stiffness            & $0.8$ \\
$\xi_{h3}$         & cubic plunge stiffness            & $\CubicPlungeStiffness$ \\
$\xi_{\alpha3}$    & cubic pitch stiffness             & $20$ \\
\midrule
$V_{f}$            & linear flutter velocity           & $\FlutterVelocity$ \\
$\omega_{f}$       & flutter angular frequency         & $\FlutterFrequency$ \\
$V$                & operating point                   & $\OperatingVelocity$ \\
$T$                & limit-cycle period                & $\LcoPeriod$ \\
$\omega$           & limit-cycle angular frequency     & $\LcoAngularFrequency$ \\
\bottomrule
\end{tabular}
\end{table}

The purpose is methodological. This is close to the smallest system containing
all the structure the question needs: a neutral phase direction, a stable
complex slow pair, and a separate faster real Floquet direction that is not
metric-normal to the slow bundle. The reduction
from four states to three isolates the choice of discarded direction and
settles nothing about model order reduction at scale.

The structural stiffness and mass blocks also supply the state metric used
throughout,
\begin{equation}
\begin{gathered}
M_{s}=\mathrm{blockdiag}\bigl(\mathrm{diag}(\xi_{h0},\xi_{\alpha0}),\,\mathcal{M}\bigr),
\eqbreak
\lVert \mathrm{d}x\rVert_{M_{s}}^{2}
 =\mathrm{d}q^{\top}\mathrm{diag}(\xi_{h0},\xi_{\alpha0})\,\mathrm{d}q
\eqsplit
 +\mathrm{d}\dot{q}^{\top}\mathcal{M}\,\mathrm{d}\dot{q},
\end{gathered}
\label{eq:metric}
\end{equation}
a dimensionless mechanical-energy form that keeps displacements and velocities
in separate blocks. It is used for the nearest-phase search, for whitening the
least-squares solves, that is, solving them in coordinates transformed by a
Cholesky factor of $M_{s}$, and for normalising perturbation amplitudes. It is
not used to define the discarded direction; that is the distinction developed
in Section~\ref{sec:theory}.

\paragraph{Validation and choice of operating point.}
Linear analysis of \eqref{eq:eom} gives a Hopf crossing at
$V_{f}=\FlutterVelocity$ with nondimensional angular frequency
$\omega_{f}=\FlutterFrequency$, against the approximate $V\approx7.6$ reported
at the resolution of the source's velocity sweep; the difference of
$\FlutterOnsetDiscrepancy$ lies inside the tolerance declared before the
comparison was run. All time integration uses the classical fourth-order
Runge--Kutta (RK4) method at fixed step. At the tested velocity below the
crossing both initial conditions quoted in the source decay to numerical zero,
and at the four tested velocities above it they converge to the same periodic
response. Reducing the integration step from $\StepCoarse$ through
$\IntegrationStep$ to $\StepFine$ changes the root-mean-square (RMS) plunge
response at $V=10$ by at most $\StepConvergenceDeviation$ relative to the
finest step.

The reported experiments use $V=\OperatingVelocity$, chosen from a scan of the
post-flutter range before any reduction was attempted. There the limit cycle is
stable and isolated and the transverse spectrum separates cleanly into one
complex pair and one faster real mode. At the higher velocities of the scan the
response becomes modulated, which would confound the comparison, and at the
neighbouring stable points the separation is less clean: at the lower one the
real mode is comparable in modulus to the complex pair, and at the higher one
the complex pair lies close to the unit circle. The separation at
$V=\OperatingVelocity$ is finite rather than extreme, which matters in both
directions: the discarded mode decays quickly enough that a three-dimensional
description is meaningful, but not so quickly that off-manifold transients pass
unobserved between samples.

\paragraph{Floquet structure.}
Shooting with a variational integrator gives a refined periodic orbit of period
$T=\LcoPeriod$, hence $\omega=\LcoAngularFrequency$
(Figure~\ref{fig:benchmark}(b)). The multipliers are
\begin{equation}
\begin{gathered}
\mu_{1}=\MultiplierNeutral,\eqbreak
\mu_{2,3}=\MultiplierSlowReal\pm\MultiplierSlowImag\,i
 \;\;(|\mu_{2,3}|=\MultiplierSlowModulus),\eqbreak
\mu_{f}=\MultiplierFast,
\end{gathered}
\label{eq:multipliers}
\end{equation}
shown in Figure~\ref{fig:benchmark}(c); the unit multiplier is reproduced to
$\NeutralMultiplierResidual$. The retained slow bundle is spanned by the phase
tangent and the real and imaginary parts of the complex mode; the discarded
direction corresponds to $\mu_{f}$, which contracts by a factor
$\MultiplierFast$ per period, so the component of a displacement along
$w_{f}$ falls to $\FastDecayFivePeriods$ of its initial size within five
periods. That rate also fixes the useful window for the future-consistency
diagnostic of Section~\ref{sec:canonicalmetric}: after one period the fast
transient is strongly reduced and after two nearly gone, while the slow
amplitude remains measurable. Table~\ref{tab:floquet} collects these
quantities. The frame \eqref{eq:frame} is well conditioned around the orbit,
with whitened condition number below $\ObliqueConditionMax$.

\begin{widefigure}
\centering
\includegraphics[width=\textwidth]{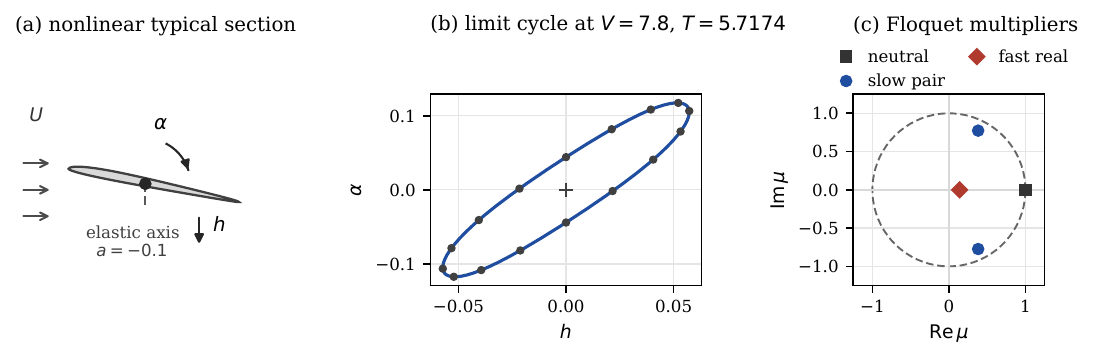}
\caption{Benchmark and operating point. (a) The nonlinear typical section of
\cite{nitti2021localized}: plunge $h$ positive downward, pitch $\alpha$
positive nose-up, elastic axis at $a=-0.1$, cubic stiffness in both degrees of
freedom. (b) The refined limit cycle at $V=\OperatingVelocity$ in the
$(h,\alpha)$ plane, with the $\PhaseCount$ reference phases at which
perturbations are placed. (c) Floquet multipliers relative to the unit circle
(dashed): one neutral, one stable complex pair of modulus
$\MultiplierSlowModulus$, and one faster real mode at $\MultiplierFast$. The
linear flutter velocity is $V_{f}=\FlutterVelocity$
(Table~\ref{tab:parameters}).}
\label{fig:benchmark}
\end{widefigure}

\begin{widetable}
\centering
\caption{Floquet structure at $V=\OperatingVelocity$ and the geometry of the two
candidate discarded directions. Angles and coefficients of $v_{\perp}$ are
measured in the metric \eqref{eq:metric} over $\DecompositionPhases$ phases.}
\label{tab:floquet}
\begin{tabular}{lr}
\toprule
Quantity & Value \\
\midrule
neutral multiplier $\mu_{1}$                       & $\MultiplierNeutral$ \\
slow pair $\mu_{2,3}$                              & $\MultiplierSlowReal\pm\MultiplierSlowImag\,i$ \\
slow modulus $|\mu_{2,3}|$                         & $\MultiplierSlowModulus$ \\
fast multiplier $\mu_{f}$                          & $\MultiplierFast$ \\
fast exponent $\lambda_{f}$                        & $\FastExponent$ \\
\midrule
whitened slow-frame condition number               & $\SlowFrameConditionMax$ \\
whitened full-frame condition number               & $\ObliqueConditionMax$ \\
periodic closure of $w_{f}$ (relative)             & $\FastModeClosure$ \\
\midrule
$\angle_{M}(v_{\perp},w_{f})$, min / median / max  & $\AngleVperpWfMin^{\circ}$ / $\AngleVperpWfMedian^{\circ}$ / $\AngleVperpWfMax^{\circ}$ \\
$|c_{\theta}|$ of unit $v_{\perp}$, median [rad]   & $\VperpPhaseCoefficientMedian$ \\
slow component of unit $v_{\perp}$, median         & $\VperpSlowNormMedian$ \\
fast component of unit $v_{\perp}$, median         & $\VperpFastCoefficientMedian$ \\
\bottomrule
\end{tabular}
\end{widetable}

\subsection{Coordinate maps}
\label{sec:methods}

All four maps below send $x\in\mathbb{R}^{4}$ to
$u=(\theta,a,b)\in S^{1}\times\mathbb{R}^{2}$, where the slow coordinates $a$
and $b$ are unrelated to the elastic-axis offset $a$ of \eqref{eq:matrices}.
They share the metric \eqref{eq:metric}, the phase origin, the amplitude
normalisation $s=\AmplitudeScale$ (a training-set amplitude scale, used in
Section~\ref{sec:metrics}), the reduced flow described below and the
evaluation protocol. Table~\ref{tab:maps} lists their structural properties and
the role each plays.

\paragraph{Anchor phase.}
Every map begins from the same anchor, the phase of the metric-nearest point of
the interpolated orbit,
\begin{equation}
\theta^{\ast}(x)\in\arg\min_{\theta\in S^{1}}
  \lVert x-\gamma(\theta)\rVert_{M_{s}}^{2},
\label{eq:anchor}
\end{equation}
solved on a periodic cubic-Hermite interpolant of $\gamma$ with
$\GridNodesPrimary$ nodes and exact node derivatives, refining every discrete
local minimum by a bracketed Newton iteration on
$\gamma'(\theta)^{\top}M_{s}(\gamma(\theta)-x)=0$. The construction presumes
that the nearest-orbit point is locally unique and depends smoothly on $x$,
which holds on a sufficiently thin tube about $\gamma$ and is checked rather
than assumed: the interpolant is periodic by construction, with a seam residual
of $\SeamClosureResidual$, and on every split each state had exactly one
discrete candidate. Equation~\eqref{eq:anchor}
is a convention, not a theorem, and Section~\ref{sec:conventionlimit} returns
to what that costs at second order. The anchor uses the metric only to select a
point on the orbit; which direction is discarded is a separate choice, and it
is that choice which distinguishes the first two maps.

\paragraph{Metric-orthogonal projection.}
The first construction takes the $M_{s}$-orthogonal complement $v_{\perp}$ of
the slow bundle as the discarded direction and solves the metric least-squares
problem for the slow coefficients,
\begin{equation}
\begin{gathered}
c=\arg\min_{c\in\mathbb{R}^{3}}
  \bigl\lVert B_{s}(\theta^{\ast})c-\bigl(x-\gamma(\theta^{\ast})\bigr)\bigr\rVert_{M_{s}},
\eqbreak
u_{\mathrm{orth}}(x)=\bigl(\theta^{\ast}+c_{\theta},\,c_{a},\,c_{b}\bigr),
\end{gathered}
\label{eq:orth}
\end{equation}
with $B_{s}$ the first three columns of \eqref{eq:frame}, solved by a whitened
QR factorisation. It differs from \eqref{eq:obliquesolve} below in one respect
only, the direction it removes. This is the construction one writes down first
when a metric is available, and it is exact on $\gamma$; we include it because
its failure mode is instructive, not as a competitive baseline. The
corresponding difficulty for slow spectral submanifolds at equilibria is set
out by Bettini et al. \cite{bettini2025oblique}.

\paragraph{Adjoint-Floquet projection.}
The second construction is the linearised phase--isostable chart of
Section~\ref{sec:firstorder}, used here as established theory
\cite{wilson2016isostable,brown2004phase}. It replaces the discarded direction
by the propagated fast Floquet mode $w_{f}$. Using the full frame
\eqref{eq:frame}, which is square and invertible, the decomposition
\begin{equation}
x-\gamma(\theta^{\ast})=B_{F}(\theta^{\ast})\,c,
\qquad c=(c_{\theta},c_{a},c_{b},c_{f}),
\label{eq:obliquesolve}
\end{equation}
is unique, and the reduced coordinate discards the fast coefficient,
\begin{equation}
u_{\mathrm{aF}}(x)=\bigl(\theta^{\ast}+c_{\theta},\,c_{a},\,c_{b}\bigr).
\label{eq:oblique}
\end{equation}
The solve is performed in whitened coordinates for conditioning, but the
whitening does not enter the definition, since \eqref{eq:obliquesolve} has a
unique metric-independent solution. The map has no fitted parameters, and the
frame it needs is a by-product of the Floquet analysis one performs anyway to
establish stability of the limit cycle. Its slow columns are the frozen gauge
used by the learned maps, so phase origin, complex-mode orientation and
amplitude scale are shared exactly. $w_{f}$ is obtained by propagating the
monodromy eigenvector of $\mu_{f}$ with the exponential factor removed, as in
\eqref{eq:floquetmodes}, and interpolating periodically at the same resolution
as $\gamma$ and $v_{c}$; its relative closure residual over one period is
$\FastModeClosure$, and its sign and scale do not affect the slow
coefficients.

\paragraph{Learned higher-order correction.}
\label{sec:correction}
The third and fourth maps add a learned remainder to \eqref{eq:oblique}, a
data-driven stand-in for the higher-order terms of the nonlinear
phase--isostable coordinates that the parameterization method computes for
systems given in closed form \cite{perezcervera2020global} and that other work
fits from data \cite{wilson2020datadriven,ahmed2023isostable}. Here the linear
part is pinned rather than fitted, so the comparison in
Section~\ref{sec:results} isolates the higher-order term. Let
$\xi=(c_{a}/s_{a},\,c_{b}/s_{b},\,c_{f}/s_{f})$ be the adjoint-Floquet
coefficients standardised by their training-set RMS values, written $\xi$ to
keep them distinct from the structural displacement $q=[h,\alpha]^{\top}$ of
\eqref{eq:eom}, and let $\rho^{2}=\lVert\xi\rVert^{2}$. The correction
is
\begin{equation}
\begin{gathered}
u_{\mathrm{IF}}(x)=u_{\mathrm{aF}}(x)+\Delta u(x),
\eqbreak
\Delta u=\tfrac12\bigl[\Delta_{\rm raw}(x)+\Delta_{\rm raw}(-x)\bigr],
\end{gathered}
\label{eq:learned}
\end{equation}
where the raw correction takes one of two forms, which we call Model~A and
Model~B. Model~A applies a common quadratic radial prefactor
$\rho^{2}$ to the network output, whereas Model~B parameterises each quadratic
term $\xi_{i}\xi_{j}$ individually,
\begin{equation}
\begin{gathered}
\text{(A)}\quad \Delta_{\rm raw}=\rho^{2}\,g(\cos\theta,\sin\theta,\xi),
\eqbreak
\text{(B)}\quad \Delta_{{\rm raw},k}=\sum_{i\le j}\xi_{i}\xi_{j}\,
   g_{k,ij}(\cos\theta,\sin\theta,\xi),
\end{gathered}
\label{eq:forms}
\end{equation}
with $g$ a single-hidden-layer network of width eight with hyperbolic-tangent
activation, taking the five inputs $(\cos\theta,\sin\theta,\xi)$ and returning
the three components of $\Delta_{\rm raw}$ for Model~A and the eighteen
coefficients $g_{k,ij}$ for Model~B; here $k$ indexes the component of $u$. In
both architectures every term carries an explicit factor quadratic in $\xi$, so
both corrections are at least quadratic near the orbit. Model~A has
$\AsixParameters$ trainable parameters and Model~B has $\BsixParameters$.

These properties hold by construction, not by penalty, and are verified
numerically in Section~\ref{sec:verification}. The quadratic prefactor forces
$\Delta u=0$ and $D\Delta u=0$ on $\gamma$, so the correction cannot alter the
first-order chart. The learned map $u_{\mathrm{IF}}$ therefore inherits from
$u_{\mathrm{aF}}$ both the pinned first-order gauge and the on-orbit fibre
condition $Du(\gamma(\theta))\,w_{f}(\theta)=0$; neither asserts anything about
an exact nonlinear fibre away from the orbit. Dependence on $\theta$ enters
only through $(\cos\theta,\sin\theta)$, so the map is periodic with no seam,
and the symmetrisation in \eqref{eq:learned} makes it exactly equivariant under
the symmetry \eqref{eq:symmetry} below. The loss contains no collapse
penalty along $v_{\perp}$, and selection uses no transverse-sensitivity
criterion. Both belonged to an earlier formulation and are inconsistent with
Section~\ref{sec:theory}, since they would require the map to identify states
whose asymptotic phases differ.

\paragraph{Discrete symmetry.}
\label{sec:symmetry}
Equation~\eqref{eq:eom} is odd, so $x\mapsto-x$ is a symmetry of the flow and
$\gamma(\theta+\pi)=-\gamma(\theta)$, which holds on the held-out data to
$\DataSymmetryResidual$. Differentiating and propagating, every column of the
frame obeys $B_{F}(\theta+\pi)=-B_{F}(\theta)$, confirmed numerically to
$\FrameAntiperiodicityResidual$. Since the displacement also changes sign,
\eqref{eq:obliquesolve} gives $c(-x)=c(x)$ and the induced action on the
reduced coordinate is
\begin{equation}
S_{z}:(\theta,a,b)\longmapsto(\theta+\pi,\,a,\,b).
\label{eq:symmetry}
\end{equation}
The adjoint-Floquet chart satisfies $u(-x)=S_{z}(u(x))$ exactly; the
symmetrisation in \eqref{eq:learned} extends this to the learned maps. We
derived \eqref{eq:symmetry} from the frame rather than assuming a sign flip of
$(a,b)$, and verified it on training states before it was imposed.

\paragraph{Shared reduced flow.}
\label{sec:reducedflow}
All four maps use the same reduced map, fixed in advance and never refitted:
\begin{equation}
\begin{gathered}
r_{\tau}(\theta,a,b)=\Bigl(\theta+\omega\tau,\;
  e^{\mathrm{Re}\lambda_{c}\tau}R(-\mathrm{Im}\lambda_{c}\tau)
  \begin{bmatrix}a\\ b\end{bmatrix}\Bigr),
\eqbreak \tau=\FlowDt,
\end{gathered}
\label{eq:reducedflow}
\end{equation}
with $R$ a planar rotation and $\lambda_{c}=\log\mu_{c}/T$ the exponent of the
retained pair, $\mu_{c}=\mu_{2}$ being the retained complex multiplier of
\eqref{eq:multipliers}. Because \eqref{eq:reducedflow} is identical for every map and is
tuned to none of them, differences in the invariance residual
\eqref{eq:invariance} are attributable to the reduction map alone. This is a
property of the present protocol, not a general principle for comparing
reduced-order models.

\paragraph{Representation freedom, and the gauge this study fixes.}
\label{sec:gauge}
A reduced description is in general equivalent under invertible changes of
reduced coordinate $\Phi$, accompanied by the corresponding conjugacy of the
reduced dynamics, $r_{\tau}\mapsto \Phi\circ r_{\tau}\circ \Phi^{-1}$. In the
present study the reduced flow is held fixed, so the coordinate freedom that
remains is the set of transformations that leave this frozen flow unchanged,
that is, those that commute with it. Writing the retained pair as a complex
amplitude $\zeta=a+ib$, \eqref{eq:reducedflow} acts on it as multiplication by
the scalar $e^{\lambda_{c}^{\ast}\tau}$, the star denoting complex conjugation.
Constant complex multiplications commute with one
another, so for every $\kappa=\lvert\kappa\rvert e^{i\chi}\in\mathbb{C}\setminus\{0\}$
the relabelling
\begin{equation}
\begin{gathered}
(\theta,\zeta)\longmapsto(\theta,\kappa \zeta),
\eqbreak
r_{\tau}(\theta,\kappa \zeta)=\bigl(\theta+\omega\tau,\;\kappa\,r_{\tau}(\zeta)\bigr),
\end{gathered}
\label{eq:gauge}
\end{equation}
leaves the reduced dynamics unchanged. The phase is not multiplied by $\kappa$:
its unit is fixed by the requirement that it advance at rate $\omega$. It keeps
the trivial freedom of origin, $\theta\mapsto\theta+\theta_{0}$, which the
adopted orbit convention fixes once for all four maps, but it has no analogue
of the complex amplitude scaling. The residual representation freedom examined
in this paper is therefore the constant complex scalar acting on the retained
slow pair.

Three distinctions follow, and the interpretation of the higher-order results
depends on keeping them apart.

\begin{enumerate}
\item \emph{Representation freedom under the frozen flow.} With the reduced
  flow held fixed, the admissible coordinate changes are those commuting with
  it, and \eqref{eq:gauge} is such a transformation. An error functional that measures
  the slow-amplitude discrepancy against a fixed reference scale, as
  \eqref{eq:residual} does, is not invariant under it, so no latent residual is
  by itself a representation-free quantity.
\item \emph{The gauge actually fixed in this experiment.} That general freedom
  is not left open here. The learned maps are pinned by construction,
  $Du_{\mathrm{IF}}(\gamma(\theta))=Du_{\mathrm{aF}}(\gamma(\theta))$,
  so the adjoint-Floquet chart and Models A and B already share one and the
  same on-orbit first-order amplitude normalisation. There is no omitted gauge
  between them to be restored, and multiplying a learned coordinate by a
  scalar $\kappa\neq1$ does not recover a shared convention: it changes the
  pinned first-order normalisation that the three maps had in common.
\item \emph{Nonlinear, amplitude-dependent change.} What $\Delta u$ contributes
  is not of the form \eqref{eq:gauge}. It vanishes with its first derivative on
  the orbit and grows with the excursion, so the ratio it induces between a
  learned amplitude and the chart amplitude tends to one near $\gamma$ and
  departs from one only off it. An amplitude-dependent contraction of that kind
  is not a constant relabelling, and Section~\ref{sec:scale} measures it as
  such.
\end{enumerate}

The first two items pull in opposite directions. Latent residuals are
representation dependent, so no single one of them is decisive; yet the maps
compared here share their first-order gauge by construction, so a fitted
rescaling is not a correction of that comparison.
Section~\ref{sec:metrics} therefore reports the pre-specified endpoint together
with disclosed sensitivity analyses and does not replace one figure of merit by
another.

\begin{widetable}
\centering
\caption{The four coordinate maps and the role each plays. All share the
metric, the phase origin, the amplitude normalisation, the reduced flow
\eqref{eq:reducedflow} and the evaluation protocol; they differ only in the
reduction map. The adjoint-Floquet chart and Models A and B additionally share
one on-orbit first-order normalisation, pinned by construction
(Section~\ref{sec:correction}); the fitted scalar of \eqref{eq:kapparestore}
used in Table~\ref{tab:results} is an external sensitivity transformation
applied afterwards, not a difference between the maps as constructed.}
\label{tab:maps}
\footnotesize
\begin{tabular}{lcll}
\toprule
Map & Param. & Discards & Evidence role \\
\midrule
metric-orthogonal & $0$ & $v_{\perp}$ & instructive contrast \\
adjoint-Floquet   & $0$ & $w_{f}$     & zero-parameter baseline; established theory \\
Model A (shared quadratic prefactor) & $\AsixParameters$  & $w_{f}$ & pre-specified primary higher-order model \\
Model B (quadratic-form correction)  & $\BsixParameters$ & $w_{f}$ & post-hoc corrective model \\
\bottomrule
\end{tabular}

\vspace{4pt}
\begin{minipage}{\linewidth}
\footnotesize\raggedright
\textit{Note.} The adjoint-Floquet chart and Models A and B satisfy the
linearised semiconjugacy relation \eqref{eq:semiconj} and the on-orbit fibre
condition $Du(\gamma(\theta))w_{f}(\theta)=0$, and all four maps are exactly
equivariant under \eqref{eq:symmetry}. The metric-orthogonal map annihilates
$v_{\perp}$ instead and so fails \eqref{eq:semiconj}.
\end{minipage}
\end{widetable}

\subsection{Data and evaluation protocol}
\label{sec:protocol}

\paragraph{Trajectory ensemble.}
Initial conditions are placed on and off the limit cycle as
\begin{equation}
\begin{gathered}
x_{0}=\gamma(\theta_{k})\pm\delta\,v_{\perp}(\theta_{k}),
\eqbreak
\delta\in\{0,\,0.0025,\,0.005,\,0.01,\,0.015,\,0.02\},
\end{gathered}
\label{eq:initial}
\end{equation}
at $\PhaseCount$ equally spaced phases $\theta_{k}=2\pi k/\PhaseCount$,
$k=0,\dots,15$, with both signs of the displacement, along the
metric-normal complement $v_{\perp}$. This was the direction available when the
ensemble was generated. It is normalised to unit metric norm, to within
$\PerturbationCalibration$, so $\delta$ is the metric length of the
displacement; by Section~\ref{sec:theory} it carries both slow and fast Floquet
content, which is what makes it a discriminating test of the reduction map; it
is not the fast Floquet direction and is not used as one. The amplitude ladder,
the phase grid and the split below were fixed before the dataset was
generated.

Phases and amplitudes are divided between the splits by whole trajectory. The
training split uses the eight even phase indices with
$\delta\in\{0,0.0025,0.01,0.02\}$, the validation split the phase indices
$1,5,9,13$ with $\delta\in\{0,0.005,0.015\}$, and the held-out split the
phase indices $3,7,11,15$ with the full ladder. Each phase contributes one
on-orbit trajectory, at $\delta=0$ where the two signs coincide, and two
trajectories per nonzero amplitude, so the splits contain
$8\times7=\TrajTrain$, $4\times5=\TrajValidation$ and $4\times11=\TrajTest$
trajectories. No trajectory, and no symmetry image of one, crosses a split
boundary, and only the held-out split spans the whole ladder.

Trajectories are integrated by RK4 at step $\IntegrationStep$, the step of the
benchmark validation, for $\TrajectoryDuration$ time units and sampled every
$\SampleInterval$, the time step $\tau$ of the reduced map. A separate pilot
run of $\PilotTrajectories$ trajectories, not part of the dataset, over a grid
of twelve phases with the five nonzero amplitudes and both signs, confirmed
that this duration suffices: every pilot trajectory returned to the same limit
cycle, the
worst mean metric distance from the orbit over the last ten time units being
$\PilotWorstTerminalDistance$, and every trajectory of the dataset meets the
same recovery criterion. Re-integrating the training initial conditions at
half the step changes each fitted slope of the two zero-parameter charts, for
both the one-step and the one-period endpoint, by less than $2\times10^{-6}$,
so the first-order scaling is not an artefact of the integration step.

Because the ensemble probes a single direction field, chosen as exactly the
complement whose use the paper examines, it is informative about that failure
mode by construction and is not a sample of general off-manifold geometry. The
first-order comparison is accordingly a targeted verification of the mechanism
in \eqref{eq:orthcoeff}, not an average-case comparison between projection
schemes. The metric-orthogonal map enters as an instructive contrast; the
baseline against which the learned maps are judged is the adjoint-Floquet
chart. The effect sizes reported below, including the relative weight of
retained phase and retained amplitude, are likewise properties of this
perturbation family, and the higher-order comparison is restricted to it: we
extrapolate neither to displacements along $w_{f}$ nor to displacements inside
the retained bundle.

\paragraph{Symmetry and the effective sample size.}
Since $\gamma(\theta+\pi)=-\gamma(\theta)$ and
$v_{\perp}(\theta+\pi)=-v_{\perp}(\theta)$, the initial conditions at phase
index $k$ and at $k+8$ with the same sign are images of each other under
$x\mapsto-x$, and so are the trajectories they generate; they are not
independent experiments. The held-out phases therefore form
$\TestDistinctPhases$ symmetry-distinct phase pairs, $(\theta_{3},\theta_{11})$
and $(\theta_{7},\theta_{15})$, and the $\TestTrajectories$ held-out
trajectories reduce to $\TestSymmetryUnits$ symmetry-distinct units: for each
pair, one on-orbit unit and one unit per sign at each of the five nonzero
amplitudes. The $\TestOffManifoldUnits$ units with $\delta>0$ are the
off-manifold units, four per amplitude. All interval estimates resample those
units rather than trajectories or time samples, by a cluster bootstrap with
$\BootstrapResamples$ resamples and a recorded seed; we report no $p$-values,
and every principal comparison below is also reported separately for each
phase pair. Because only $\TestDistinctPhases$ symmetry-distinct phases are
represented, these intervals describe variability across the tested
symmetry-distinct conditions and are not interpreted as population-level
confidence intervals for generalisation over phase.

\paragraph{One pre-specified endpoint and three disclosed sensitivities.}
\label{sec:metrics}
The pre-specified endpoint is the latent residual defined next, that is, the
semiconjugacy residual measured in the reduced coordinates, under the
normalisation fixed before training and recorded before held-out evaluation.
Everything after it is a post-hoc sensitivity analysis of that endpoint,
introduced because a latent residual is representation dependent
(Section~\ref{sec:gauge}). We report all
four and do not promote any sensitivity to the status of a more canonical
measurement: each fixes a different convention, and none supplies an
independent nonlinear reference coordinate.

\paragraph{The pre-specified scale-sensitive latent metric.}
The quantity minimised in training is the normalised residual of
\eqref{eq:invariance} at the sampling interval $\tau=\FlowDt$. For a pair of
consecutive samples,
\begin{equation}
\begin{gathered}
e^{2}(x)=\tfrac13\Bigl[\,2-2\cos\bigl(\theta^{+}-\hat{\theta}\bigr)
\eqsplitdelim
  +\bigl((a^{+}-\hat{a})^{2}+(b^{+}-\hat{b})^{2}\bigr)/s^{2}\Bigr],
\eqbreak
(\hat{\theta},\hat{a},\hat{b})=r_{\tau}\bigl(u(x)\bigr),
\end{gathered}
\label{eq:residual}
\end{equation}
with $(\theta^{+},a^{+},b^{+})=u(\varphi^{\tau}(x))$ and $s=\AmplitudeScale$ the
training-set amplitude scale. Finite-horizon accuracy uses the same
normalisation: for a horizon of $m$ periods we encode a state, iterate
\eqref{eq:reducedflow} for $\lceil mT/\tau\rfloor$ steps and compare with the
encoding of the true future state, reporting the latent distance divided by
$\sqrt{3}$; $m=1$ and $m=5$ were fixed before any comparison was run. The
amplitude term of \eqref{eq:residual} is not invariant under \eqref{eq:gauge}.
We retain the metric because it is the one under which the models were
trained, selected and first evaluated, label it the \emph{scale-sensitive}
metric, and rest no claim about a learned map on it alone.

One further point bears on the one-step endpoint: it is the training objective,
so the learned maps are optimised against the same functional that scores them,
and a held-out split controls for overfitting to particular trajectories but
not for that alignment. We treat one-step invariance as a training-consistent
diagnostic and take one-period and multi-period prediction as the endpoints
that carry evidence.

\paragraph{Validation-fitted global amplitude recalibration.}
For each learned map we fit, on validation data alone, the constant complex
scalar that best expresses its slow coordinate in
the amplitude units of the adjoint-Floquet chart,
\begin{equation}
\begin{gathered}
\kappa=\arg\min_{\kappa\in\mathbb{C}}
  \bigl\lVert \zeta_{\mathrm{aF}}-\kappa\,\zeta_{\mathrm{learned}}\bigr\rVert^{2},
\eqbreak
\zeta^{\kappa}_{\mathrm{learned}}=\kappa\,\zeta_{\mathrm{learned}},
\end{gathered}
\label{eq:kapparestore}
\end{equation}
and recompute every endpoint on $\zeta^{\kappa}$ with the phase untouched. By
\eqref{eq:gauge} the recalibration changes no reduced dynamics. It is
\emph{not} a restoration of a gauge the maps had failed to share: they already
share their pinned on-orbit first-order normalisation, and for $\kappa\neq1$
this transformation alters that normalisation instead of recovering it. It is a
sensitivity analysis, showing how far the numerical ranking under
\eqref{eq:residual} moves when the amplitude unit of a learned map is rescaled
globally by a single fitted constant. The transformation, the whitening
covariances of the secondary check below, the metric definitions and the
evaluator were hashed and frozen before the held-out trajectories were
reopened. The analysis is nonetheless \emph{post hoc}, motivated by review of
the original result, and it is reported at that evidential weight throughout.

\paragraph{Adjoint-Floquet-targeted future consistency.}
\label{sec:canonicalmetric}
A learned map still appears on both sides of a rollout error, encoding the
initial state and also the future state used as the target, so a map that
contracts its coordinate shrinks both. The diagnostic that avoids this encodes
once and never again. For an initial held-out state $x_{0}$ we form
$z_{0}=u(x_{0})$ from the map under test, optionally apply the disclosed
recalibration \eqref{eq:kapparestore}, propagate the frozen reduced flow to
$z_{\rm pred}(t)=r_{t}(z_{0})$, and compare against
\begin{equation}
z_{\rm target}(t)=u_{\mathrm{aF}}\bigl(\varphi^{t}(x_{0})\bigr),
\label{eq:canonical}
\end{equation}
the adjoint-Floquet coordinates of the true future state; the learned encoder
is never applied at the endpoint. We evaluate at one, two and five periods, the
first two being where the question has content by the fast decay rate of
Section~\ref{sec:benchmark}, and report the phase and slow-amplitude components
separately, the phase coordinate having no analogue of the scalar freedom
\eqref{eq:gauge} once its rate and common origin are fixed.

The diagnostic therefore tests whether a map's initial coordinate improves the
prediction of the future full-order state when that state is expressed in the
common adjoint-Floquet chart. Its advantages are specific: the learned encoder
appears only at the initial state, a contraction of the coordinate cannot
shrink prediction and target together, and all maps are compared in the same
frozen chart units. It is also reference asymmetric, by construction. The
target is supplied by the adjoint-Floquet chart itself, not by an independently
computed exact nonlinear phase--isostable coordinate, so the quantity measured
is whether a map's reduced prediction stays consistent with future full states
\emph{as read by the linear chart}. A map that assigns coordinates differently
from that chart off the orbit is penalised whether or not its assignment is
closer to the exact nonlinear quotient. The diagnostic is therefore legitimate
as a chart-consistency test and as a guard against endpoint re-encoding, but it
is not independent ground truth and does not establish which map is closest to
the exact nonlinear phase--isostable coordinate. It is likewise post hoc.

\paragraph{A secondary check.}
As an independent control we re-express the residual of \eqref{eq:residual} in
each map's own validation-frozen whitened amplitude units, keeping the circular
phase term separate. Whitening by a general symmetric matrix does not commute
with \eqref{eq:reducedflow}, so it is applied to the residual and never to the
state; this is a consistency check on the recalibration of
\eqref{eq:kapparestore}, not a third primary endpoint.

\paragraph{Two quantities deliberately absent.}
Reconstruction error is undefined here, since none of the maps has a decoder. A
transverse-collapse metric, used in earlier work on this benchmark, is excluded
as evidence: because $v_{\perp}$ is not the fast Floquet direction,
Section~\ref{sec:theory} implies that states at $\pm\delta v_{\perp}$ should be
separated in the reduced coordinate, not identified, so the metric rewards the
wrong behaviour.

\subsection{Training and model selection}
\label{sec:selection}

The learned maps were trained on the residual \eqref{eq:residual} over training
trajectories only, with the Adam optimiser \cite{kingma2015adam} at learning
rate $\TrainLearningRate$, batch size $\TrainBatchSize$, $\TrainIterations$
iterations, a fixed random seed and a ridge penalty of weight $\RidgeWeight$ on
the network parameters, retaining the checkpoint with the lowest validation
residual. These settings were fixed in the training configuration before
training, and no hyperparameter search was run for the two maps reported here.
A pre-specified heuristic selection rule, fixed in advance, governed selection
between Models A and B: prefer the
larger model only if it improves the validation residual by at least five per
cent without degrading any structural diagnostic by more than ten per cent.

That rule selected Model~A, for a reason that turned out to be spurious.
Model~B improved the validation residual by $\BsixValidationImprovement$ per
cent, but the relative ten-per-cent clause was being applied to diagnostics
whose exact value is zero by construction. Their measured values, of order
$\AsixFirstOrderResidual$ for the first-derivative diagnostics, are
finite-difference truncation errors: for both models they decrease at second
order as the finite-difference step is reduced, and a relative comparison
between two such values is meaningless. We identified this only after Model~A
had been evaluated on the held-out set, so the amendment is post hoc. The
revised rule treats zero-target diagnostics below a numerical floor as
equivalent. The floors were set from the reference geometry and validation data
alone, never from held-out performance: $\FloorFirstOrder$ for the two
derivative diagnostics, one per cent of their acceptance threshold fixed before
training; $\FloorSymmetry$ for the symmetry diagnostic, ten times the
adjoint-Floquet chart's own asymmetry on validation data; and double-precision
round-off for the on-orbit correction. Under the revised rule Model~B is
selected, from an unmodified checkpoint never evaluated on held-out data.

Model~A therefore remains the pre-specified primary higher-order model, while
Model~B is retained as a disclosed post-hoc corrective sensitivity whose
held-out numbers carry the weight appropriate to a post-hoc selection. That
provenance ordering is fixed and is
used throughout. It records which result was specified in advance, not which
interpretation is correct: as Section~\ref{sec:discussion} sets out, the
pre-specified Model~A endpoint is measured in units that partly reward a
smaller coordinate, so a favourable value there is not by itself a confirmed
benefit, and the post-hoc diagnostics that qualify it are not a confirmed null.

%% file: sections/04_results.tex
\section{Results}
\label{sec:results}

\subsection{Verification and first-order geometry}
\label{sec:verification}

\paragraph{Verification of the constructions.}
The numerical checks confirm that every coordinate map has, to numerical
accuracy, the structural properties on which the comparison relies. Each map
reproduces the limit cycle: the maximum reference phase error is
$\ObliqueLcoPhaseError$ for the adjoint-Floquet chart and
$\OrthogonalLcoPhaseError$ for the orthogonal one, the standardised amplitudes
vanish to $\ObliqueLcoAmplitude$, and coarsening the orbit interpolant from
$\GridNodesPrimary$ to $\GridNodesConvergence$ nodes changes the chart
coordinates by at most $\ObliqueGridConvergenceRms$ in RMS. For both learned
models the correction vanishes on the orbit, to at most $\AsixLcoCorrection$ in
standardised units, and so does its first derivative: along all four frame
directions the finite-difference derivative at step $10^{-5}$ is
$\AsixFirstOrderResidual$ for Model~A and about twice that for Model~B, below
the floor $\FloorFirstOrder$ in both cases, and for both it falls at second
order as the step is reduced, by two orders of magnitude for a tenfold
reduction, as the truncation error of the difference quotient does. The
fast-direction derivative $Du\,w_{f}$ behaves identically, at
$\AsixFastFiberResidual$ for Model~A. These values confirm numerically what the
quadratic prefactor of \eqref{eq:forms} guarantees structurally. For Model~A a
power law fitted to the magnitude of the correction gives exponents
$p=\AsixCorrectionScalingSynthetic$ for synthetic displacements and
$p=\AsixCorrectionScalingData$ across the perturbation ladder, consistent with a
leading term at least quadratic in the excursion, though a fitted exponent of
this kind does not by itself establish exact quadratic homogeneity. The
equivariance \eqref{eq:symmetry} holds for both learned models with maximum
defects below the chart's own $\ObliqueSymmetryAmplitudeMax$ in standardised
amplitude, which sets the floor; for Model~B the maximum is
$\BsixSymmetryAmplitudeMax$.
These checks rule out violations of the imposed on-orbit, first-order,
fast-direction and symmetry constraints as explanations of the differences
analysed below. They verify the constructions as specified; they say nothing
about how closely the learned correction approximates the exact nonlinear
phase--isostable coordinate.

\paragraph{The discarded direction: metric normal against fast Floquet direction.}
\label{sec:geometry}
Figure~\ref{fig:geometry} shows that the two candidate complements are far
apart everywhere on the orbit. Measured in the metric \eqref{eq:metric} over
$\DecompositionPhases$ phases, the angle between $v_{\perp}$ and $w_{f}$ ranges
from $\AngleVperpWfMin^{\circ}$ to $\AngleVperpWfMax^{\circ}$ with median
$\AngleVperpWfMedian^{\circ}$ (Figure~\ref{fig:geometry}(b)), never approaching
zero, so the metric-normal complement is nowhere a good approximation to the
fast Floquet direction on this orbit.

The consequence is visible in the decomposition \eqref{eq:vperpdecomp}
(Figure~\ref{fig:geometry}(c)). A unit $v_{\perp}$ has a retained component of
metric norm $\VperpSlowNormMedian$ (median), comparable to the discarded fast
component $\VperpFastCoefficientMedian$, so a displacement along $v_{\perp}$
carries first-order content that the adjoint-Floquet chart keeps and the
orthogonal map annihilates. That retained content has both a phase and a
slow-amplitude part. The phase coefficient $|c_{\theta}|$ ranges from
$\VperpPhaseCoefficientMin$ to $\VperpPhaseCoefficientMax$ radians per unit
$\delta$, with median $\VperpPhaseCoefficientMedian$; it is the infinitesimal
phase response curve evaluated along $v_{\perp}$ \cite{brown2004phase}, a
standard quantity whose size for this system is what we report. At
$\delta=0.02$ it shifts the asymptotic phase of each of the states at
$\pm\delta v_{\perp}$, to first order and in opposite directions, by about
$\VperpPhaseShiftLargestDelta$ radians at the median, placing them on
different isochrons. The retained-amplitude part is not negligible either. We
do not rank the two. As Section~\ref{sec:firstorder} notes, $c_{\theta}$ and
$(c_{a},c_{b})$ multiply basis vectors with different physical normalisations,
and the residual \eqref{eq:residual} weights the amplitude difference by $1/s$;
a claim about which retained component dominates would therefore be a claim
about the units adopted rather than about the geometry.

Figure~\ref{fig:scaling} shows what this does to the reduced coordinate: the
metric-orthogonal projection leaves a first-order error and the adjoint-Floquet
projection an approximately second-order one. The orthogonal projection incurs
an error growing as $\delta^{\SlopeOrthogonalInvariance}$ in the invariance
residual and $\delta^{\SlopeOrthogonalOnePeriod}$ in one-period rollout;
replacing the discarded direction by $w_{f}$ changes the exponents to
$\SlopeObliqueInvariance$ and $\SlopeObliqueOnePeriod$, as the first-order
coefficients \eqref{eq:orthcoeff} and \eqref{eq:semiconj} predict. The
prediction that the adjoint-Floquet chart would remove the linear term and
leave $O(\delta^{2})$ was specified, with an expected slope near two, and
recorded before the held-out set was opened; slopes of
$\ObliqueTrainSlopeInvariance$ and $\ObliqueValidationSlopeInvariance$ had been
obtained on training and validation data at that time. In pooled terms the
invariance residual falls from $\TestOrthogonalInvariance$ to
$\TestObliqueInvariance$ and the one-period error from
$\TestOrthogonalOnePeriod$ to $\TestObliqueOnePeriod$, with no fitted
parameters in either map. Neither this result nor the geometry behind it
involves a learned quantity, and neither is affected by anything in the
remainder of this section. We therefore treat the adjoint-Floquet chart, not
the orthogonal one, as the baseline against which learning must be judged.

\begin{widefigure}
\centering
\includegraphics[width=\textwidth]{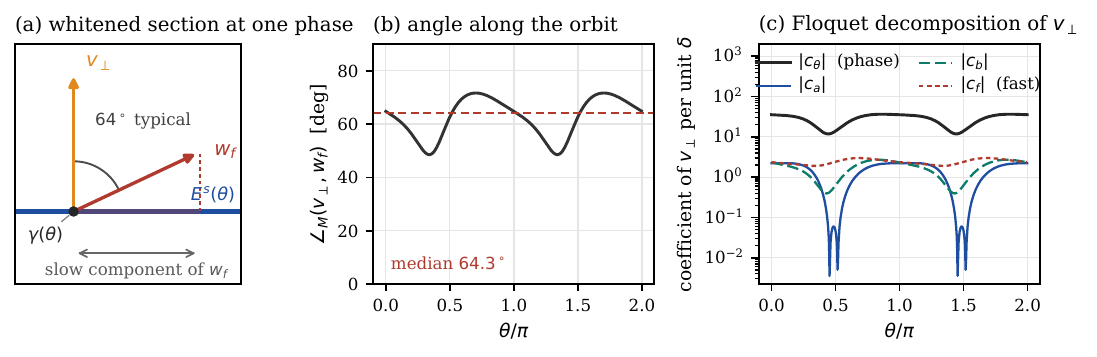}
\caption{The metric normal is not the fast Floquet direction. (a) The
two-dimensional section of the whitened state space containing $v_{\perp}$ and
$w_{f}$ at one phase, drawn to scale; the slow bundle meets this section in the
horizontal line. (b) The metric angle between the two directions around the
orbit. (c) Coefficients of a unit $v_{\perp}$ in the Floquet frame
\eqref{eq:frame}: retained content is present in both the phase and the
slow-amplitude directions. Those coefficients multiply differently normalised
basis vectors, so their raw magnitudes are not directly comparable and the
panel is not evidence that either retained component dominates.}
\label{fig:geometry}
\end{widefigure}

\begin{widefigure}
\centering
\includegraphics[width=0.94\textwidth]{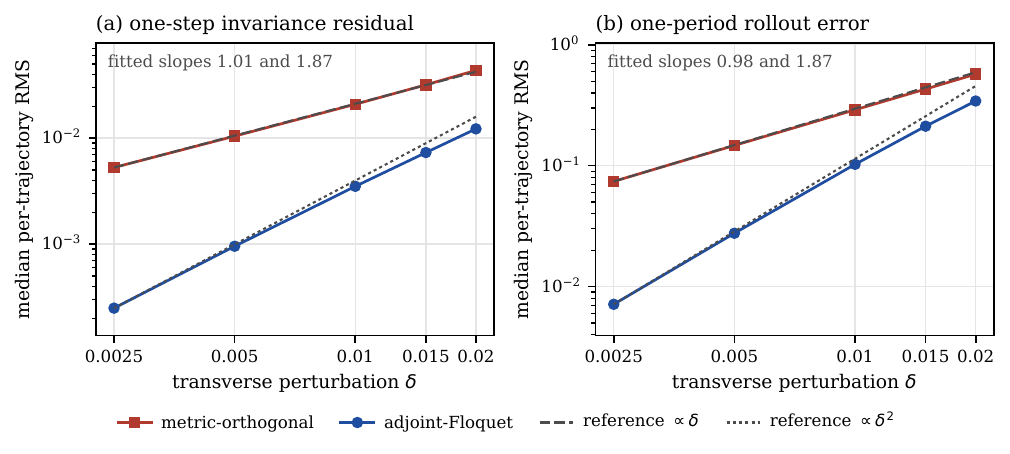}
\caption{Error scaling of the two zero-parameter charts against the transverse
amplitude $\delta$ on held-out data. Markers joined by solid lines are the
measured medians; the dashed and dotted lines are reference slopes proportional
to $\delta$ and $\delta^{2}$, and each panel states the fitted slopes.
Replacing the metric-normal direction by the fast Floquet direction as the
discarded direction changes the exponent from about one to about two for both
the one-step invariance residual (a) and the one-period rollout error (b). The
$O(\delta^{2})$ expectation was specified before the held-out set was opened.}
\label{fig:scaling}
\end{widefigure}

\subsection{Higher-order corrections and metric sensitivity}
\label{sec:scale}

\paragraph{The pre-specified endpoint.}
Under the pre-specified scale-sensitive metric both learned maps improve on the
adjoint-Floquet baseline at the short horizons (Table~\ref{tab:results}, first
column). Over the $\TestOffManifoldUnits$ off-manifold units the paired median
error ratio to the chart is $\PairedAsixObliqueInvarianceInverse$ for Model~A
and $\PairedBsixObliqueInvarianceInverse$ for Model~B on the one-step residual,
and $\PairedAsixObliqueOnePeriodInverse$ and
$\PairedBsixObliqueOnePeriodInverse$ at one period. At five periods the ratios
are $\PairedAsixObliqueFivePeriodInverse$ and
$\PairedBsixObliqueFivePeriodInverse$, so neither map improves on the chart
there. Fitted against $\delta$ in the same way as in Figure~\ref{fig:scaling},
the held-out invariance residual scales as $\delta^{\SlopeAsixInvariance}$ for
Model~A and $\delta^{\SlopeBsixInvariance}$ for Model~B.

\begin{widetable}
\centering
\caption{Held-out comparison of the learned maps against the zero-parameter
adjoint-Floquet chart on $\TestOffManifoldUnits$ symmetry-distinct
off-manifold units. Entries are the paired median error ratio of the learned
map to the chart, so values below one favour the learned map, with the
cluster-bootstrap $95\%$ confidence interval (CI) on the median $\log_{10}$
gain and the number of units favouring the learned map. The upper block reports
the pre-specified endpoint alongside two post-hoc sensitivities of it: a
validation-fitted global amplitude recalibration \eqref{eq:kapparestore} and
the self-whitened check, both disclosed post hoc. The lower block reports the
adjoint-Floquet-targeted future-consistency diagnostic
\eqref{eq:canonical}, which is reference asymmetric because the baseline chart
supplies the target coordinates. None of the three post-hoc columns is a more
canonical measurement of the pre-specified quantity, and none is designated the
primary control; Section~\ref{sec:discussion} states what they jointly do and
do not identify.}
\label{tab:results}
\footnotesize
\begin{tabular}{llccc}
\toprule
 & & \textbf{Pre-specified} & Recalibrated & Self-whitened \\
Endpoint & Map & (prospective) & (post hoc) & (post hoc) \\
\midrule
\multirow{2}{*}{one-step invariance}
 & Model A (primary)    & $\PairedAsixObliqueInvarianceInverse$ & $\RestoredAsixInvarianceRatio$ $[\RestoredAsixInvarianceCiLow,\,\RestoredAsixInvarianceCiHigh]$, $\RestoredAsixInvarianceUnits/\TestOffManifoldUnits$ & $\WhitenedAsixInvarianceRatio$ \\
 & Model B (corrective) & $\PairedBsixObliqueInvarianceInverse$ & $\RestoredBsixInvarianceRatio$ $[\RestoredBsixInvarianceCiLow,\,\RestoredBsixInvarianceCiHigh]$, $\RestoredBsixInvarianceUnits/\TestOffManifoldUnits$ & $\WhitenedBsixInvarianceRatio$ \\
\midrule
\multirow{2}{*}{one-period rollout}
 & Model A (primary)    & $\PairedAsixObliqueOnePeriodInverse$ & $\RestoredAsixOnePeriodRatio$ $[\RestoredAsixOnePeriodCiLow,\,\RestoredAsixOnePeriodCiHigh]$, $\RestoredAsixOnePeriodUnits/\TestOffManifoldUnits$ & $\WhitenedAsixOnePeriodRatio$ \\
 & Model B (corrective) & $\PairedBsixObliqueOnePeriodInverse$ & $\RestoredBsixOnePeriodRatio$ $[\RestoredBsixOnePeriodCiLow,\,\RestoredBsixOnePeriodCiHigh]$, $\RestoredBsixOnePeriodUnits/\TestOffManifoldUnits$ & $\WhitenedBsixOnePeriodRatio$ \\
\midrule
\multirow{2}{*}{five-period rollout}
 & Model A (primary)    & $\PairedAsixObliqueFivePeriodInverse$ & $\RestoredAsixFivePeriodRatio$ $[\RestoredAsixFivePeriodCiLow,\,\RestoredAsixFivePeriodCiHigh]$, $\RestoredAsixFivePeriodUnits/\TestOffManifoldUnits$ & $\WhitenedAsixFivePeriodRatio$ \\
 & Model B (corrective) & $\PairedBsixObliqueFivePeriodInverse$ & $\RestoredBsixFivePeriodRatio$ $[\RestoredBsixFivePeriodCiLow,\,\RestoredBsixFivePeriodCiHigh]$, $\RestoredBsixFivePeriodUnits/\TestOffManifoldUnits$ & $\WhitenedBsixFivePeriodRatio$ \\
\bottomrule
\end{tabular}

\vspace{4pt}
\begin{tabular}{llcc}
\toprule
\multicolumn{4}{l}{\emph{Adjoint-Floquet-targeted future consistency} (post hoc; reference asymmetric:} \\
\multicolumn{4}{l}{encode once, propagate, compare against future states encoded by the baseline chart)} \\
Horizon & Map & ratio $[95\%$ CI$]$, units & phase component only \\
\midrule
\multirow{2}{*}{one period}
 & Model A (primary)    & $\CanonicalAsixOnePeriodRatio$ $[\CanonicalAsixOnePeriodCiLow,\,\CanonicalAsixOnePeriodCiHigh]$, $\CanonicalAsixOnePeriodUnits/\TestOffManifoldUnits$ & $\CanonicalPhaseAsixOnePeriodRatio$ \\
 & Model B (corrective) & $\CanonicalBsixOnePeriodRatio$ $[\CanonicalBsixOnePeriodCiLow,\,\CanonicalBsixOnePeriodCiHigh]$, $\CanonicalBsixOnePeriodUnits/\TestOffManifoldUnits$ & $\CanonicalPhaseBsixOnePeriodRatio$ \\
\midrule
\multirow{2}{*}{two periods}
 & Model A (primary)    & $\CanonicalAsixTwoPeriodRatio$ $[\CanonicalAsixTwoPeriodCiLow,\,\CanonicalAsixTwoPeriodCiHigh]$, $\CanonicalAsixTwoPeriodUnits/\TestOffManifoldUnits$ & $\CanonicalPhaseAsixTwoPeriodRatio$ \\
 & Model B (corrective) & $\CanonicalBsixTwoPeriodRatio$ $[\CanonicalBsixTwoPeriodCiLow,\,\CanonicalBsixTwoPeriodCiHigh]$, $\CanonicalBsixTwoPeriodUnits/\TestOffManifoldUnits$ & $\CanonicalPhaseBsixTwoPeriodRatio$ \\
\midrule
\multirow{2}{*}{five periods}
 & Model A (primary)    & $\CanonicalAsixFivePeriodRatio$ $[\CanonicalAsixFivePeriodCiLow,\,\CanonicalAsixFivePeriodCiHigh]$, $\CanonicalAsixFivePeriodUnits/\TestOffManifoldUnits$ & $\CanonicalPhaseAsixFivePeriodRatio$ \\
 & Model B (corrective) & $\CanonicalBsixFivePeriodRatio$ $[\CanonicalBsixFivePeriodCiLow,\,\CanonicalBsixFivePeriodCiHigh]$, $\CanonicalBsixFivePeriodUnits/\TestOffManifoldUnits$ & $\CanonicalPhaseBsixFivePeriodRatio$ \\
\bottomrule
\end{tabular}

\vspace{4pt}
\begin{minipage}{\linewidth}
\footnotesize\raggedright
\textit{Note.} Wherever an interval is shown, an entry reads
ratio $[95\%$ CI$]$, units favouring the learned map out of
$\TestOffManifoldUnits$; the remaining columns give the ratio alone. Ratios
below one favour the learned map.
\end{minipage}
\end{widetable}

\paragraph{The scale of the learned coordinates.}
Figure~\ref{fig:scalegauge}(a) shows that the learned slow coordinates occupy a
smaller amplitude range than the chart's over the validation ensemble while
keeping the shape of its covariance ellipse. The RMS standardised
slow-coordinate norm is $\ScaleChartRms$ for the adjoint-Floquet chart,
$\ScaleAsixRms$ for Model~A and $\ScaleBsixRms$ for Model~B, with anisotropy
ratios $\ScaleChartAnisotropy$, $\ScaleAsixAnisotropy$ and
$\ScaleBsixAnisotropy$. Fitting \eqref{eq:kapparestore} on $\GaugeFitSamples$
validation samples gives $|\kappa|=\GaugeAsixAlpha$ with rotation
$\chi=\GaugeAsixRotation^{\circ}$ for Model~A, accounting for a fraction
$\GaugeAsixExplained$ of the variance of the chart coordinate, and
$|\kappa|=\GaugeBsixAlpha$ at $\chi=\GaugeBsixRotation^{\circ}$ for Model~B,
accounting for $\GaugeBsixExplained$. The corresponding amplitude ratios
$1/|\kappa|$ are $\GaugeAsixGain$ and $\GaugeBsixGain$.

Figure~\ref{fig:scalegauge}(b) shows that this contraction is amplitude
dependent and transient, and so not a transformation of the form
\eqref{eq:gauge}. The median local ratio
$|\zeta_{\rm learned}|/|\zeta_{\mathrm{aF}}|$ over all validation samples is
$\LocalAsixMedian$ for Model~A and $\LocalBsixMedian$ for Model~B, almost
exactly one; resolved by time since release it is $\LocalAsixFirstPeriod$ and
$\LocalBsixFirstPeriod$ during the first period and returns to
$\LocalAsixSettled$ and $\LocalBsixSettled$ once trajectories settle. The maps
contract where the excursion is large and approach the chart near the orbit,
which is what a correction pinned to $Du_{\mathrm{aF}}$ on $\gamma$ and growing
at least quadratically away from it must do. Describing either map as applying
a constant rescaling would therefore misstate its own measured behaviour.

Figure~\ref{fig:scalegauge}(c) shows that the single fitted scalars
nevertheless bound how much of the measured gain the residual can attribute to
geometry unaided. Scaling the chart's slow coordinate by the reciprocal
fitted amplitude factor $1/|\kappa|$ of each model, which contracts it towards
that model's amplitude range and changes no dynamics whatever, already
reproduces
$\ScaleFractionAsixInvariance$ of Model~A's measured one-step validation gain
and $\ScaleFractionAsixOnePeriod$ of its one-period gain, with
$\ScaleFractionBsixInvariance$ and $\ScaleFractionBsixOnePeriod$ for Model~B.
Read correctly this is a statement about the metric, not about the maps: a
comparable share of the gain is available from a pure change of amplitude unit,
so the residual on its own cannot separate a better assignment from a smaller
coordinate. It does not follow, and the preceding paragraph shows it is not
true, that the learned maps act as such a change of unit.

\paragraph{Sensitivity to validation-fitted amplitude recalibration.}
\label{sec:controlled}
Figure~\ref{fig:controlled}(a) shows that the recalibration
\eqref{eq:kapparestore} moves the held-out ranking substantially; open markers
give the pre-specified endpoint and filled markers the recalibrated one.
Against the adjoint-Floquet chart, the paired median error ratio of Model~A
over the $\TestOffManifoldUnits$ off-manifold units becomes
$\RestoredAsixInvarianceRatio$ for the one-step residual, interval
$[\RestoredAsixInvarianceCiLow,\RestoredAsixInvarianceCiHigh]$ on the
$\log_{10}$ gain with Model~A lower on $\RestoredAsixInvarianceUnits$ of
$\TestOffManifoldUnits$ units; $\RestoredAsixOnePeriodRatio$ at one period,
$[\RestoredAsixOnePeriodCiLow,\RestoredAsixOnePeriodCiHigh]$,
$\RestoredAsixOnePeriodUnits$ units; and $\RestoredAsixFivePeriodRatio$ at five
periods, $[\RestoredAsixFivePeriodCiLow,\RestoredAsixFivePeriodCiHigh]$,
$\RestoredAsixFivePeriodUnits$ units. The first two intervals contain zero and
the third excludes it against Model~A. The self-whitened check moves the same
way, at $\WhitenedAsixInvarianceRatio$, $\WhitenedAsixOnePeriodRatio$ and
$\WhitenedAsixFivePeriodRatio$, so the shift is not peculiar to one choice of
rescaling.

For Model~B the recalibrated ratios are $\RestoredBsixInvarianceRatio$ at one
step ($[\RestoredBsixInvarianceCiLow,\RestoredBsixInvarianceCiHigh]$,
$\RestoredBsixInvarianceUnits$ units), $\RestoredBsixOnePeriodRatio$ at one
period ($[\RestoredBsixOnePeriodCiLow,\RestoredBsixOnePeriodCiHigh]$,
$\RestoredBsixOnePeriodUnits$ units) and $\RestoredBsixFivePeriodRatio$ at
five, with the self-whitened check giving $\WhitenedBsixInvarianceRatio$ and
$\WhitenedBsixOnePeriodRatio$. The point estimates still favour Model~B at the
short endpoints, by rather less than the pre-specified metric reported, and
both intervals now contain zero.

The ranking under \eqref{eq:residual} is therefore not robust to a global
change of amplitude unit. The recalibrated ranking is not thereby the correct
one, because the recalibration alters the first-order normalisation that the
maps share (Section~\ref{sec:metrics}).

\begin{widefigure}
\centering
\includegraphics[width=\textwidth]{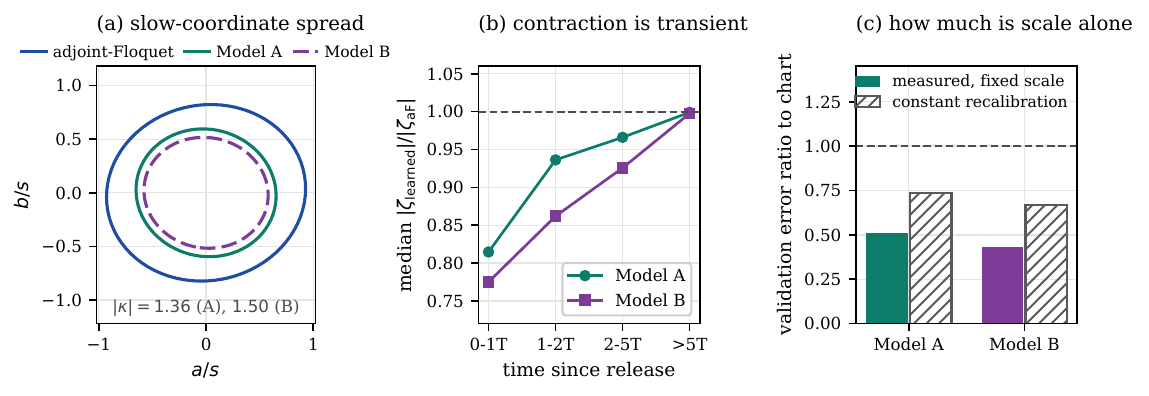}
\caption{Sensitivity of latent-error rankings to amplitude scaling, on
validation data. (a) One-standard-deviation ellipses of the standardised slow
coordinate. (b) Median local ratio $|\zeta_{\rm learned}|/|\zeta_{\rm aF}|$
against time since release. (c) Measured error ratio to the adjoint-Floquet
chart under the fixed normalisation (solid bars), against the ratio produced
by scaling the chart's slow coordinate by the reciprocal validation-fitted
factor $1/|\kappa|$ alone (hatched bars).
Panel (c) bounds how much of the measured gain a pure change of amplitude unit
can reproduce; it does not identify the learned maps with such a change, which
panel (b) rules out. Values are given in Section~\ref{sec:scale}.}
\label{fig:scalegauge}
\end{widefigure}

\begin{widefigure}
\centering
\includegraphics[width=\textwidth]{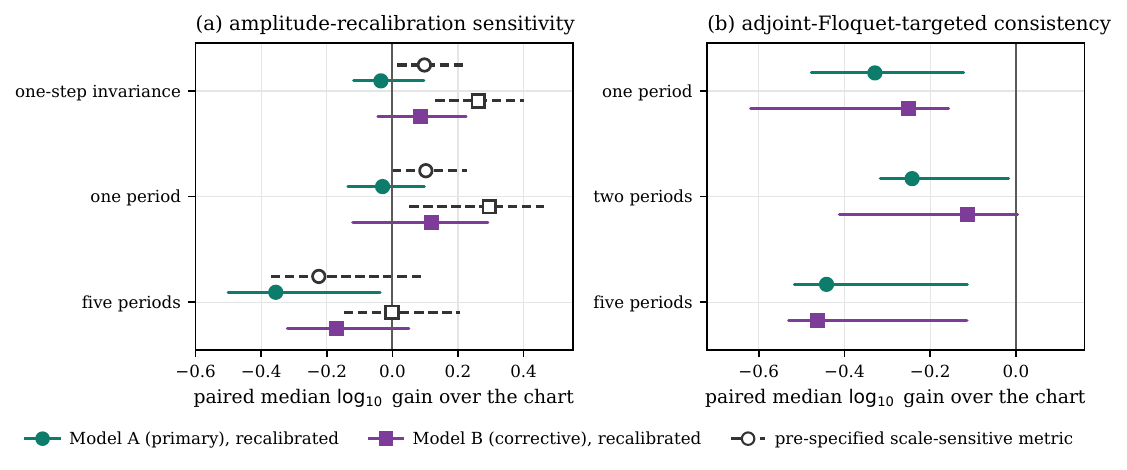}
\caption{Held-out comparison and its post-hoc sensitivities. Markers are paired
median $\log_{10}$ gains over the adjoint-Floquet chart across the
$\TestOffManifoldUnits$ symmetry-distinct off-manifold units, with
cluster-bootstrap $95\%$ intervals; positive favours the learned map. (a) The
pre-specified scale-sensitive endpoint (open markers, dashed intervals) against
the same endpoints after the global recalibration \eqref{eq:kapparestore}
(filled markers). (b) Future consistency against targets supplied by the
adjoint-Floquet chart, which makes the panel reference asymmetric
(Section~\ref{sec:canonicalmetric}); neither panel is an independent ranking of
predictive accuracy. Model~A is the prospectively specified primary model and
Model~B the disclosed corrective one; both panels are post-hoc sensitivity
analyses.}
\label{fig:controlled}
\end{widefigure}

\subsection{Future consistency and robustness}
\label{sec:canonical}

\paragraph{Future consistency relative to the linear adjoint-Floquet chart.}
Figure~\ref{fig:controlled}(b) shows that, when the future target coordinates
are supplied by the adjoint-Floquet chart itself
(Section~\ref{sec:canonicalmetric}), both learned maps are less consistent with
those targets than the chart is. For Model~A the paired median ratio is
$\CanonicalAsixOnePeriodRatio$ at one period
($[\CanonicalAsixOnePeriodCiLow,\CanonicalAsixOnePeriodCiHigh]$, with
$\CanonicalAsixOnePeriodUnits$ of $\TestOffManifoldUnits$ units favouring it),
$\CanonicalAsixTwoPeriodRatio$ at two
($[\CanonicalAsixTwoPeriodCiLow,\CanonicalAsixTwoPeriodCiHigh]$) and
$\CanonicalAsixFivePeriodRatio$ at five
($[\CanonicalAsixFivePeriodCiLow,\CanonicalAsixFivePeriodCiHigh]$). For
Model~B they are $\CanonicalBsixOnePeriodRatio$
($[\CanonicalBsixOnePeriodCiLow,\CanonicalBsixOnePeriodCiHigh]$, with
$\CanonicalBsixOnePeriodUnits$ of $\TestOffManifoldUnits$ units favouring it),
$\CanonicalBsixTwoPeriodRatio$ and $\CanonicalBsixFivePeriodRatio$. Every
interval except Model~B's at two periods excludes zero, in each case against
the learned map.

Separating the two components of \eqref{eq:canonical} locates the
disagreement mainly in the slow amplitude: at one and two periods the amplitude
ratios differ from the combined ones by less than $0.01$, and at five periods
they are larger. The phase component, which has no analogue of the scalar
freedom \eqref{eq:gauge}, is the one place where a learned map agrees with the
target better: Model~B reaches ratio $\CanonicalPhaseBsixOnePeriodRatio$ at one
period ($[\CanonicalPhaseBsixOnePeriodCiLow,\CanonicalPhaseBsixOnePeriodCiHigh]$)
and $\CanonicalPhaseBsixTwoPeriodRatio$ at two
($[\CanonicalPhaseBsixTwoPeriodCiLow,\CanonicalPhaseBsixTwoPeriodCiHigh]$), the
interval excluding zero only at two periods, against
$\CanonicalPhaseAsixOnePeriodRatio$ and $\CanonicalPhaseAsixTwoPeriodRatio$ for
Model~A.

Because the chart supplies the target, these figures measure consistency with
future states \emph{as encoded by the linear chart}. They show that the learned
maps depart from the chart in a way that does not shrink under propagation;
they are not an independent ranking of predictive accuracy.

\paragraph{Dependence on the perturbation phase.}
\label{sec:phasedep}
Figure~\ref{fig:phase} shows that the two symmetry-distinct held-out phase
pairs disagree about Model~A. With only $\TestDistinctPhases$ such pairs in the
held-out set, every statement above pools two groups that need not agree.
Reading the recalibrated one-step ratio at the five nonzero amplitudes
individually, the first pair, $(\theta_{3},\theta_{11})$, gives
$\RestoredAsixInvariancePhaseADeltaA$,
$\RestoredAsixInvariancePhaseADeltaB$, $\RestoredAsixInvariancePhaseADeltaC$,
$\RestoredAsixInvariancePhaseADeltaD$ and
$\RestoredAsixInvariancePhaseADeltaE$, above one throughout and drifting
upward, while the second, $(\theta_{7},\theta_{15})$, gives
$\RestoredAsixInvariancePhaseBDeltaA$, $\RestoredAsixInvariancePhaseBDeltaB$,
$\RestoredAsixInvariancePhaseBDeltaC$, $\RestoredAsixInvariancePhaseBDeltaD$
and $\RestoredAsixInvariancePhaseBDeltaE$, falling below one at the larger
amplitudes. Model~A's apparent amplitude dependence is thus carried by one of
the two groups and reverses in the other: it does not replicate. Two groups are
enough to withhold any amplitude-threshold or phase-independent claim, which we
accordingly do not make, and too few to distinguish systematic phase structure
of the correction from sampling variability. We use the disagreement only in
the first, negative sense.

Model~B behaves consistently on this endpoint, its ratios running
$\RestoredBsixInvariancePhaseADeltaA$ through
$\RestoredBsixInvariancePhaseADeltaE$ for the first pair and
$\RestoredBsixInvariancePhaseBDeltaA$ through
$\RestoredBsixInvariancePhaseBDeltaE$ for the second, both falling below one at
the three larger amplitudes. Its phase-prediction advantage does not reproduce:
under future consistency targeted on the chart the ratio is
$\CanonicalPhaseBsixOnePeriodPhaseA$ against
$\CanonicalPhaseBsixOnePeriodPhaseB$ at one period and
$\CanonicalPhaseBsixTwoPeriodPhaseA$ against
$\CanonicalPhaseBsixTwoPeriodPhaseB$ at two, so the one effect whose interval
excluded zero under a scale-independent component also fails to replicate
across the two groups.

\paragraph{Dependence on the phase-anchor convention.}
\label{sec:conventionlimit}
The decomposition into a fixed linear chart plus a learned increment is not
unique at second order, because the anchor \eqref{eq:anchor} is a convention.
Replacing metric-nearest anchoring by a Floquet fixed-point convention leaves
the first-order chart unchanged and moves the coordinate at second order, on
validation data by a median $\ConventionPhaseMedian$ in phase and
$\ConventionAmplitudeMedian$ in standardised amplitude, against learned
corrections of median magnitude $\CorrectionAsixPhaseMedian$ and
$\CorrectionAsixAmplitudeMedian$ for Model~A and $\CorrectionBsixPhaseMedian$
and $\CorrectionBsixAmplitudeMedian$ for Model~B. For Model~A the ambiguity is
the same size as the correction and exceeds it in phase; for Model~B the
correction is larger by about a factor of two but the same order. The learned
increment $\Delta u$ therefore cannot be read as a unique second-order
foliation coefficient independently of the anchoring convention. This bears on
the interpretation of $\Delta u$, not on the validity of the complete maps:
$u_{\mathrm{IF}}$ remains a well-defined coordinate map for each fixed
convention, and whether the complete maps predict better than the chart is a
separate question, addressed by the endpoints above and left unresolved by
them.

\begin{widefigure}
\centering
\includegraphics[width=0.88\textwidth]{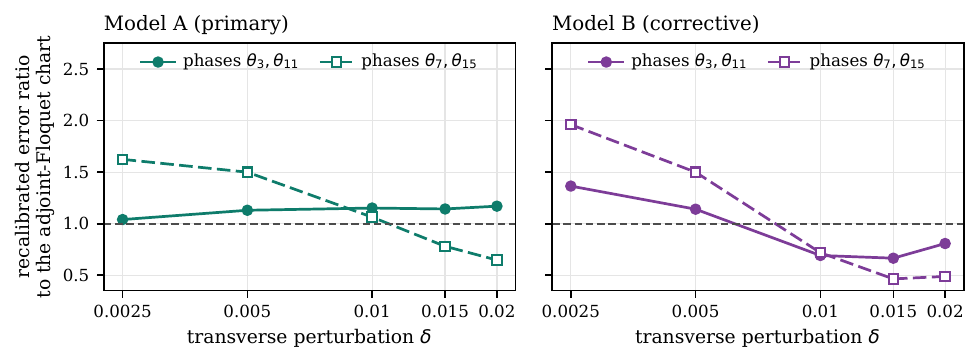}
\caption{Recalibrated one-step invariance ratio against the transverse
amplitude $\delta$, separated by the two symmetry-distinct held-out phase pairs
$(\theta_{3},\theta_{11})$ and $(\theta_{7},\theta_{15})$ (solid and dashed
lines); values below one favour the learned map. For Model~A the two groups
disagree and the trend does not replicate; for Model~B both fall below one at
the three larger amplitudes. The panel is descriptive only: two
symmetry-distinct phases are too few to establish general phase structure.}
\label{fig:phase}
\end{widefigure}

%% file: sections/05_discussion.tex
\section{Discussion and scope}
\label{sec:discussion}

\subsection{First-order geometry and relation to earlier work}
\label{sec:discfirst}

Proposition~\ref{prop:firstorder} isolates what is at stake. A reduced
coordinate inherits an $O(\delta^{2})$ invariance residual precisely when its
linearisation intertwines the linearised flow with the reduced map; a
projection that annihilates some other complement generically fails that
relation and pays an $O(\delta)$ price. Nothing in the condition refers to
angles, and a metric on state space is a modelling convenience, so there is no
reason to expect the $M$-orthogonal complement of the retained bundle to be the
invariant one. The statement is local and generic in the two senses made
explicit in Section~\ref{sec:semiconj}: it holds on the neighbourhood where
the expansion is valid, and a particular displacement direction may still
escape the first-order term by accidental cancellation.

The symptom is easy to misread. An orthogonal projection is exact on the orbit,
well conditioned, and passes every consistency check one would naturally apply
to it. Its failure appears only off the orbit, where it looks like a deficiency
of \emph{geometric projection in general}. A learned map compared against it
will appear to succeed by a wide margin, and on this benchmark most of that
margin is the linear term that the adjoint-Floquet chart removes at no
parameter cost.

On this limit cycle the two candidate complements are
$\AngleVperpWfMedian^{\circ}$ apart at the median and never closer than
$\AngleVperpWfMin^{\circ}$, and the retained part of a unit metric normal has
metric norm $\VperpSlowNormMedian$ against a discarded fast part of
$\VperpFastCoefficientMedian$. That retained part has two components. Its
phase coefficient is the infinitesimal phase response curve along $v_{\perp}$,
and because one retained direction is neutral this part of the mis-assignment
is not contracted by the linear flow. The slow-amplitude part is contracted, at
the rate of the retained pair, but is not small. Which of the two dominates a
measured residual depends on the units in which phase and amplitude are
compared, and we make no normalisation-free claim about that ranking.

The observed exponents, $\delta^{\SlopeOrthogonalInvariance}$ for the
metric-orthogonal chart against $\delta^{\SlopeObliqueInvariance}$ for the
adjoint-Floquet chart, agree with the first- and second-order behaviour that
Proposition~\ref{prop:firstorder} predicts for this system; the second
exponent is approximately, not exactly, two. They were obtained with no fitted
parameters in either map and predicted before the held-out set was opened. The
result stands on its own: it involves no learned quantity, no choice of latent
amplitude unit beyond the one both maps already share, and no reference chart
other than each map's own coordinates, so nothing in
Section~\ref{sec:discidentifiability} qualifies it. Its scope is nonetheless
local. It says nothing about states far from the orbit, where the leaves curve,
or about systems whose spectral gap is too small for the bundles to separate.

No spectral-submanifold model was computed here, and nothing above bears on
spectral-submanifold methodology; nor would it be right to draw the contrast as
manifold against quotient. Any reduction to an attracting manifold needs a rule
for off-manifold states. The principle we rely on, that the rule should follow
the stable foliation and not a convenient projection, has the long history
outlined in Section~\ref{sec:introduction}, from centre- and slow-manifold
initial conditions
\cite{roberts1989appropriate,cox1995initial,roberts2000computer,roberts2015emergent}
to recent spectral-submanifold work \cite{bettini2025oblique,bettini2026general}.
We claim no novelty for that principle. A closely related observation appears
in reductions by elimination of fast variables, where fast relaxation produces
a finite correction to the effective slow initial condition, the
\emph{initial slip}
\cite{vankampen1985elimination,haake1983initialslip,cox1995initial}. Such
corrections can become asymptotically small in idealised large-separation
limits while remaining appreciable at finite separation, and the present result
is consistent with that picture; the benchmark has no formal small parameter,
and we do not make the connection quantitative.

Nor is the correct first-order projection expensive here. Reconstructing the
full nonlinear or global foliation can be involved, but its linear-in-distance
approximation is not. At an equilibrium its leading order is given by the
adjoint eigenvectors of the critical modes, and higher orders follow from an
iteration of the same kind as the one that constructs the manifold, without a
full normal-form transformation \cite{roberts2000computer},
\cite[Section~12.3]{roberts2015emergent}. For a periodic orbit the leading
order is the dual Floquet frame, which the stability analysis of the cycle
already supplies. The adjoint-Floquet chart therefore needs no fitted
parameters and essentially no computation beyond the Floquet analysis, and the
computational difficulty lies at higher order. What this work adds is the
specialisation of the principle to a stable periodic orbit, where the retained
bundle contains a neutral direction and the required first-order projection is
available directly from the primal Floquet frame and its dual; the
quantification of the resulting discrepancy on an aeroelastic limit cycle; the
pre-specified scaling test; and the higher-order identifiability audit.

Closer to the application, Koopman models of supercritical flutter have
identified parameterised isostables and isochrons for an academic example and a
panel \cite{song2026koopman}, and spectral-submanifold reduction has been
applied to a forced pitch-and-plunge airfoil subject to flutter for dynamical
integrity analysis \cite{habib2025integrity}. Neither study targets controlled
perturbations transverse to a limit cycle. The present study isolates the
coordinate-assignment problem and does not compare against other data-driven
reduction approaches, such as sparse identification of the governing equations
\cite{brunton2016sindy}, dynamics-based autoencoders
\cite{mogharabin2025autoencoder} or post-bifurcation forecasting
\cite{ghadami2017postbif}.

\subsection{Higher-order identifiability}
\label{sec:discidentifiability}

A latent residual measured against a fixed amplitude unit is not a
representation-free quantity, since with the reduced flow frozen every constant
complex scaling of the retained pair is an admissible relabelling
(Section~\ref{sec:gauge}). Here the effect is not small: rescaling the chart
coordinate by the reciprocal $1/|\kappa|$ of a single fitted constant already
reproduces between
$\ScaleFractionAsixInvariance$ and $\ScaleFractionBsixOnePeriod$ of the
measured validation gain. No single latent residual therefore settles the
comparison. It does not follow that the learned maps are gauge artefacts. Their
linearisation is pinned to $Du_{\mathrm{aF}}$ on the orbit, and what
$\Delta u$ contributes is an amplitude-dependent contraction that tends to one
near the orbit, with median local ratio $\LocalAsixMedian$ for Model~A over the
validation ensemble against $\LocalAsixFirstPeriod$ during the first period
after release (Section~\ref{sec:scale}). The fitted $\kappa$ summarises that
effect compactly and shows how far the ranking moves under a global change of
unit, but the recalibrated numbers are a disclosed sensitivity, not a corrected
measurement.

Encoding once and comparing propagated coordinates against future full states
removes a real defect of a rollout error, in which the map under test supplies
both sides of the comparison so that a contracted coordinate shrinks both.
That is why the learned maps' departure from the chart appears in the
future-consistency diagnostic instead of cancelling. Because its target is
produced by $u_{\mathrm{aF}}$ itself, however, it cannot arbitrate between the
chart and a candidate correction to the chart. We read it as a
chart-consistency result: the learned assignments are less consistent with
future states encoded in the adjoint-Floquet chart, at every horizon tested.
We do not read it as showing that the learned maps are dynamically inferior.

Assembling the evidence gives a consistent but inconclusive picture. Under the
pre-specified fixed-normalisation residual the learned maps improve on the chart
on held-out data at the short horizons. Under the validation-fitted global
recalibration that improvement shrinks or reverses. Under future consistency
targeted on the adjoint-Floquet chart the linear chart is ahead. The
phase-resolved comparison does not replicate across the two symmetry-distinct
groups available. The anchor convention shifts the learned increment by an
amount comparable to the increment itself, so $\Delta u$ is not a uniquely
defined second-order object. And no diagnostic here supplies an independent
nonlinear reference coordinate, so the coordinate-assignment error of
Section~\ref{sec:semiconj}, as opposed to the semiconjugacy residual and the
prediction error, is never measured.

The conclusion is an identifiability statement, not a result in either
direction. The pre-specified comparison favours the learned maps but is
measured in units that partly reward a smaller coordinate; each diagnostic that
removes that sensitivity either changes the pinned normalisation or adopts the
baseline as its reference. We therefore report that the additional dynamical
value of the learned higher-order correction is not determined by the
available evidence, and we decline to convert post-hoc diagnostics into a
confirmatory null. Settling it requires something the present design does not
contain: an independently computed second-order phase--isostable expansion to
serve as a reference, or an endpoint invariant by construction under the
admissible transformations.

\subsection{Scope and limitations}
\label{sec:limitations}

\paragraph{Scope of the evidence.}
The evidence comes from one benchmark at one operating point with one
reduction: the isolated airfoil of \cite{nitti2021localized} at
$V=\OperatingVelocity$, $\xi_{h3}=\CubicPlungeStiffness$, reduced from four
states to three. The aerodynamics are quasi-steady, without lag states, and a
system with aerodynamic memory or a smaller spectral gap could behave
differently. The study isolates the coordinate-assignment problem; forced or
gust response, control, and dependence on airspeed or nonlinearity level are
outside its scope. The empirical magnitudes and higher-order findings reported
here are specific to this configuration; Proposition~\ref{prop:firstorder}
itself applies under the assumptions stated in Section~\ref{sec:semiconj}.

\paragraph{Perturbation family and amplitude range.}
Off-manifold states were generated along a single direction, $v_{\perp}$, at
five nonzero amplitudes up to $\delta=0.02$. A perturbation aligned with $w_{f}$
or with a slow mode would test different aspects of the reduction map. Nothing
here extrapolates beyond $\delta=0.02$, and no amplitude threshold is claimed.

\paragraph{Phase coverage.}
The held-out trajectories come from only $\TestDistinctPhases$
symmetry-distinct phase pairs (Section~\ref{sec:protocol}), so per-amplitude
statements rest on four units each. The two pairs disagree about Model~A
(Section~\ref{sec:phasedep}), and with two groups there is no way to tell
whether that is sampling noise or genuine phase structure of the correction;
for the same reason the cluster-bootstrap intervals describe the tested
conditions, not generalisation over phase. Phase coverage, not trajectory
count, is the binding constraint on this dataset.

\paragraph{Fixed reduced dynamics.}
The reduced map \eqref{eq:reducedflow} is the frozen linear Floquet map for all
four coordinate maps. This isolates the reduction map, and it is also a
restriction: the multi-period results in particular may say as much about the
shared reduced dynamics as about the coordinates.

\paragraph{Evidential status.}
The first-order scaling test carries the confirmatory weight of the study.
Model~A provides the pre-specified primary higher-order endpoint, whose
interpretation is qualified by the sensitivity analyses of
Sections~\ref{sec:scale} and~\ref{sec:canonical}. The amended selection rule that promoted
Model~B was identified after Model~A had been evaluated on held-out data
(Section~\ref{sec:selection}), and the recalibration and future-consistency
analyses of Sections~\ref{sec:controlled} and~\ref{sec:canonical} are post hoc,
although their transformations were fitted on validation data and hashed before
the held-out trajectories were reopened. An earlier learned map, which was
trained with a collapse penalty along $v_{\perp}$ that
Section~\ref{sec:theory} shows to be inconsistent and which was not
equivariant under \eqref{eq:symmetry}, is excluded from the evidence.

\paragraph{Reference coordinate and anchor convention.}
No exact nonlinear phase--isostable coordinate is computed for this system, so
the coordinate-assignment error is never measured and each higher-order
diagnostic fixes a convention of its own. The decomposition into a linear chart
plus a learned second-order increment also depends on the phase-anchor
convention, by an amount comparable to the learned corrections themselves
(Section~\ref{sec:conventionlimit}), so we make no uniqueness claim for what
the networks represent. A second-order phase--isostable expansion computed by
the parameterization method \cite{perezcervera2020global} would be the most
informative comparator for the learned correction.

%% file: sections/06_conclusions.tex
\section{Conclusions}
\label{sec:conclusions}

For the aeroelastic limit cycle considered here, the choice of discarded
direction governs whether a reduced coordinate satisfies the linearised
semiconjugacy condition, and hence whether its invariance residual for
off-manifold states is of first or of second order. The metric-orthogonal complement of the retained slow Floquet bundle
is not the fast Floquet direction tangent to the strong-stable fibre on this
benchmark: the two differ by a
median $\AngleVperpWfMedian^{\circ}$, so the metric-normal projection fails the
linearised semiconjugacy condition and, by
Proposition~\ref{prop:firstorder}, generically leaves a residual of order
$\delta$ in the transverse displacement. The retained content it discards is
both phase and slow amplitude; the phase part is not contracted by the linear
flow, and which part dominates the measured residual depends on the coordinate
units adopted. Projecting instead along the fast Floquet direction, as the
classical adjoint-Floquet phase--isostable chart does, removes the first-order term
at no parameter cost and leaves the approximately quadratic residual predicted
in advance.

Once that first-order geometry is imposed, learned corrections built to vanish
with their first derivative on the orbit, to respect the discrete symmetry
exactly and to share the frozen reduced flow do reduce the pre-specified
fixed-normalisation residual on held-out trajectories. Two post-hoc
diagnostics qualify that reading. A validation-fitted global amplitude
recalibration moves the ranking substantially, showing that a latent residual
measured against a fixed unit is not representation free, and a
future-consistency test whose target coordinates are supplied by the
adjoint-Floquet chart places both learned maps behind the chart at every
horizon tested. Neither settles the question in the opposite direction: the
recalibration alters the first-order normalisation that the maps already share,
and the future-consistency target is the baseline's own reading of the future
state. The phase-resolved results do not replicate across the two
symmetry-distinct groups available, and the anchor convention displaces the
learned increment by an amount comparable to the increment.

For the tested aeroelastic limit cycle, satisfying the first-order Floquet
relation is necessary to eliminate the generic $O(\delta)$ semiconjugacy defect
under the frozen reduced dynamics, and the chart that does so is available from
the Floquet analysis at no fitting cost. Within this benchmark, operating point, perturbation family and
amplitude range, no confirmatory scale-independent benefit of the learned
higher-order correction was demonstrated, and whether it improves the
underlying nonlinear phase--isostable assignment remains unidentified by the
available, representation-dependent diagnostics. We do not extend that
statement to invariant foliations in general, to other flutter systems or
periodic orbits, or to higher-order phase--isostable corrections as such.